%% file: PM-HT-19.tex
\documentclass{article}
\usepackage[letterpaper]{geometry}

\def\withcolors{0}
\def\withnotes{0}

\input{preamble}
\begin{document}

\title{RATQ: A Universal Fixed-Length Quantizer for Stochastic
  Optimization}

\author{Prathamesh Mayekar$^\dag$ \and Himanshu Tyagi$^\dag$ } \date{}
\maketitle {\renewcommand{\thefootnote}{}\footnotetext{
\noindent$^\dag$Department of Electrical Communication Engineering,
Indian Institute of Science, Bangalore 560012, India.  Email:
\{prathamesh, htyagi\}@iisc.ac.in }}

\maketitle \renewcommand{\thefootnote}{\arabic{footnote}}
\setcounter{footnote}{0} \maketitle

\input{abstract}

\newpage
\tableofcontents
\newpage
\input{intro}
\input{prelim}

\input{almostsure}

\input{meansquare}

\input{proof}

\input{applications}

\section*{Acknowledgement} Authors would like to thank Jayadev Acharya and Ananda Theertha Suresh for useful discussions, and Alexander Barg for sharing his knowledge on covering codes and pointing to related work.
\newpage
\appendix
\appendixpage 
\input{appendix}


\newpage
 \bibliography{IEEEabrv,tit2018} 
\bibliographystyle{IEEEtranS} 
\end{document}

%% file: preamble.tex
\usepackage{varwidth}
\usepackage[utf8]{inputenc} 
\usepackage[T1]{fontenc} 
\usepackage{hyperref}
\usepackage{url} 
\usepackage{booktabs} 
\usepackage{amsfonts} 
\usepackage{nicefrac}
\usepackage{microtype} 
\usepackage[noend]{algpseudocode}
\usepackage{soul}
\usepackage[T1]{fontenc}
\usepackage{lmodern}
\usepackage{appendix}

\usepackage[usenames,dvipsnames,table]{xcolor}
\usepackage{amsmath}
\usepackage{algorithm}
\usepackage[noend]{algpseudocode}
\algsetblock[Name]{Encoder}{Stop}{3}{1cm}
\algsetblock[Name]{Decoder}{Stop}{3}{1cm}
\algsetblock[Name]{SharedRandomness}{Stop}{3}{1cm}

\usepackage[colorinlistoftodos,textsize=scriptsize]{todonotes}
\usepackage[mathscr]{eucal}

\usepackage[noadjust]{cite}
\usepackage{subcaption}
\usepackage{%
    amssymb,%
    amsthm,
    bbm,%
    etoolbox,%
    mathtools,%
    stmaryrd%
}

\usepackage{amsfonts}
\usepackage{amssymb}
\usepackage{enumitem}
\newlist{steps}{enumerate}{1}
\setlist[steps, 1]{label = \quad \quad \quad Step \arabic*:}
\usepackage{tikz,pgf,pgfplots,circuitikz}
\usepgflibrary{shapes}
\usetikzlibrary{%
    arrows,%
    decorations.text,%
    positioning,%
    scopes,%
    shapes,
    patterns%
}

\theoremstyle{plain} \newtheorem{thm}{Theorem}[section]
\newtheorem{lem}[thm]{Lemma} 
\newtheorem{cor}[thm]{Corollary}

\theoremstyle{definition} \newtheorem{defn}[thm]{Definition}

\theoremstyle{remark} \newtheorem{rem}{Remark}

\definecolor{lightgray}{gray}{0.9}

\newcommand{\ep}{\mathcal{E}}
\newcommand{\eqdef}{:=}
\newcommand{\eq}[1]{\begin{align*}#1\end{align*}}
\newcommand{\ceil}[1]{\left\lceil#1\right\rceil}%
\newcommand{\norm}[1]{\|#1\|}%

 \newcommand{\R}{\mathbb{R}}
\newcommand{\N}{\mathbb{N}}
\newcommand{\E}[1]{\mathbb{E}\left[#1\right]}

 \newcommand{\B}{\mathscr{B}}

\newcommand{\V}{\mathcal{V}}
\newcommand{\X}{\mathcal{X}}

\newcommand{\oO}{\mathcal{O}}
\newcommand{\Q}{\mathcal{Q}}
\newcommand{\uRoman}[1]{\uppercase\expandafter{\romannumeral#1}}

\DeclarePairedDelimiter\floor{\lfloor}{\rfloor}

\usepackage{subcaption}
\DeclareCaptionSubType*{algorithm}

\DeclareCaptionLabelFormat{alglabel}{Alg.~#2}

\ifnum\withcolors=1
  \newcommand{\new}[1]{{\color{red} {#1}}} 
  \newcommand{\newer}[1]{{\color{brown} {#1}}} 
  \newcommand{\newest}[1]{{\color{blue} {#1}}} 
  \newcommand{\mcolor}[1]{{\color{ForestGreen}#1}} 
  \newcommand{\tcolor}[1]{{\color{Orange}#1}} 
\else
  \newcommand{\new}[1]{{{#1}}}
  \newcommand{\newer}[1]{{{#1}}}
  \newcommand{\newest}[1]{{{#1}}}
  \newcommand{\mcolor}[1]{{#1}}
  
  \newcommand{\tcolor}[1]{{#1}}
\fi

\ifnum\withnotes=1
  \newcommand{\mnote}[1]{\par\mcolor{\textbf{M: }\sf #1}} 
  \newcommand{\tnote}[1]{\par\tcolor{\textbf{T: }\sf #1}} 
  
\else
  \newcommand{\mnote}[1]{}
  \newcommand{\tnote}[1]{}
  
\fi
\newcommand{\ignore}[1]{\leavevmode\unskip} 

\newcommand{\Qenc}{Q^{{\tt e}}}
\newcommand{\Qdec}{Q^{{\tt d}}}

\newcommand{\indic}[1]{\mathbbm{1}_{#1}}
\newcommand{\iid}{$iid$}

%% file: abstract.tex
\begin{abstract}

We present Rotated Adaptive Tetra-iterated Quantizer (RATQ), a
fixed-length quantizer for gradients in first order stochastic
optimization.  RATQ is easy to implement and involves only a Hadamard transform computation and adaptive uniform quantization with appropriately chosen dynamic ranges. For noisy gradients with almost surely bounded 
Euclidean norms, we establish an information
theoretic lower bound for optimization accuracy using finite precision
gradients and show that RATQ almost attains this lower bound. 

 For mean square bounded noisy gradients, we use a                                                          
   gain-shape quantizer which separately quantizes the Euclidean norm
 and uses RATQ to quantize the normalized unit norm vector. We
 establish lower bounds for performance of any optimization procedure
 and shape quantizer, when used with a uniform gain
 quantizer. Finally, we propose an adaptive quantizer for gain which
 when used with RATQ for shape quantizer outperforms uniform gain
 quantization and is, in fact, close to optimal.

As a by-product, we show that our fixed-length quantizer
RATQ has almost the same performance as the optimal variable-length
quantizers for distributed mean estimation.
Also, we obtain an efficient quantizer for Gaussian vectors
 which attains a rate very close to the
Gaussian rate-distortion function and is, in fact, universal for subgaussian input vectors.
\end{abstract}

%% file: intro.tex
\section{Introduction}
Stochastic gradient descent (SGD) and its variants are popular
optimization methods for machine learning. In its basic form, SGD
performs iterations $x_{t+1}=x_t - \eta \hat{g}(x_t)$, where
$\hat{g}(x)$ is a noisy estimate of the subgradient of the function
being optimized at $x$.  Our focus in this work is on a distributed
implementation of this algorithm where the output $\hat{g}(x)$ of the
first order oracle must be quantized to a precision of $r$ bits.  This
abstraction models important scenarios ranging from distributed
optimization to federated learning, and is of independent theoretical
interest.

We study the tradeoff between the convergence rate of first order
optimization algorithms and the precision $r$ available per
subgradient update.  We consider two {\em oracle models}: the first
where the subgradient estimate's Euclidean norm is {\em almost surely
  bounded} and the second where it is {\em mean square bounded}.  Our
main contributions include new quantizers for the two oracle models
and theoretical insights into the limitations imposed by heavy-tailed
gradient distributions admitted under the mean square bounded
oracles. A more specific description of our results and their relation
to prior work is provided below.

 \subsection{Prior work}\label{ss:Prior_work}
 SGD and the oracle model abstraction for it appeared in classic
 works~\cite{robbins1951stochastic} and~\cite{nemirovsky1983problem},
 respectively.  We refer the reader to textbooks and
 monographs~\cite{nemirovski1995information,nesterov2013introductory,bubeck2015convex}
 for a review of the basic setup. Recently, variants of this problem
 with quantization or communication constraints on oracle output have
 received a lot of attention~\cite{de2015taming,suresh2017distributed,
   alistarh2017qsgd, stich2018sparsified, wang2018atomo,
   agarwal2018cpsgd, acharya2019distributed, karimireddy2019error,
   basu2019qsparse, gupta2015deep, wen2017terngrad,
   ramezani2019nuqsgd}.  Our work is motivated by the results
 in~\cite{suresh2017distributed,alistarh2017qsgd}, and we elaborate on
 the connection.

Specifically,~\cite{alistarh2017qsgd} considers a problem very similar
to ours. The paper~\cite{suresh2017distributed} considers the related
problem of distributed mean estimation, but the quantizer and its
analysis is directly applicable to distributed optimization. The two
papers present different quantizers that encode each input using a
variable number of bits. Both these quantizers are of optimal expected
precision for almost surely bounded oracles. However, their worst-case
(fixed-length) performance is suboptimal.

In fact, the problem of designing fixed-length quantizers for almost
surely bounded oracles is closely related to designing small-size
covering for the Euclidean unit ball. There has been a longstanding
interest in this problem in the vector quantization and information
theory literature ($cf.$~\cite{Wyner1967,gersho2012vector, hughes1987gaussian, csiszar1991capacity, Lapidoth1997TIT, Dumer2007}).
A closely related problem is that of Gaussian rate-distortion
  where we seek to quantize a random Gaussian vector to within a
  specified mean squared error, while using as few bits per dimension
  as possible ($cf.$~\cite{gallager1968information, CovTho06}). 
Typical
  fixed-length schemes for this problem draw on its duality with the
  channel coding problem and modify channel codes to obtain coverings;
  see, for instance,~\cite{martinian2006low, SommerFederShalvi08, yan2013polar}. 
However, for
  our application of distributed optimization, these schemes may not
  be acceptable for two reasons: First the resulting complexity is
  still too high for hardware implementation; and second, the
  resulting schemes are not universal and are tied to Gaussian
  distributions specifically.

\newest{To the best of our knowledge, the spherical covering code construction with the least known encoding complexity is from~\cite{HamkinsZeger97,HamkinsZeger97ii,HamkinsZeger02} (see~\cite{HamkinsThesis} for more details). However, its performance rate-distortion 
and computational complexity have been analyzed only for the Gaussian source; it has been shown in~\cite{HamkinsThesis} that the computational complexity needed grows linearly in rate. In comparison, the variant of our scheme for the Gaussian source
is especially simple and has only constant computational complexity per dimension.
}

\new{In a slightly different direction, a seminal, but perhaps not so
  widely known, result of~\cite{ziv1985universal} provides a very simple
  universal quantizer for random vectors with independent and
  identically distributed (\iid) coordinates, with each coordinate
  almost surely bounded. In this scheme, we first quantize each
  coordinate uniformly, separately using a ``scalar-quantizer,'' and
  then apply a universal entropic compression scheme to the quantized
  vector. We note that the variable-length schemes proposed
  in~\cite{alistarh2017qsgd,suresh2017distributed} are very similar,
  albeit with a specific choice of the entropic compression scheme.}

\new{All these schemes are variable-length schemes, while it is
  desirable to get a fixed-length scheme for the ease of both protocol
  and hardware implementation.  We remark that
  indeed~\cite{suresh2017distributed} presents an interesting
  randomly-rotate and quantize fixed-length scheme, but it still
  requires communicating $O(\log \log d)$ times more than the optimal
  fixed-length quantizer for the unit Euclidean ball given
  in~\cite{Wyner1967}.  To the best of our knowledge, prior to our
  work, the quantizer in~\cite{suresh2017distributed} is the best
  known efficient fixed-length quantizer for the unit Euclidean ball.}

\newest{In fact, a randomized orthogonal transform scheme similar to that in~\cite{suresh2017distributed} appeared
almost concurrently in~\cite{HadadErez16} as well, where an  analysis for Gaussian source is presented.
  However, a rate-distortion analysis has not been done in~\cite{HadadErez16}.
Remarkably, an early instance of the ``rotated dithering'' scheme for distributing energy equally appears in
the image compression literature in~\cite{OstromoukhovHA94}, albeit without formal error or performance analysis.
Another interesting scheme was proposed in~\cite{AkyolRose13} where 
nonuniform quantization (using {\em companding}) was combined with dithering. 
Our adaptive choice of dynamic range for uniform quantizers is similar, in essence, to companding.
But our scheme differs from the one in~\cite{AkyolRose13} in several ways: 
First,~\cite{AkyolRose13} uses the knowledge of input distribution to design 
their companding function, whereas we only need knowledge of the tail behaviour of
the input distribution in our setting; second, we apply a random rotation to our input leading to a universal quantizer, which is not needed in \cite{AkyolRose13}; and finally, the specific structure of our quantizer with adaptive dynamic ranges makes it amenable to mean square error analysis for a large variety of sources.}

\newest{Nevertheless, our proposed scheme has elements of all these approaches. We build on ideas similar to these works and develop a new approach for relating the mean square error to the dynamic range of our adaptive uniform quantizer. Interestingly, it leads to new results even for the well-studied Gaussian rate-distortion problem. Namely, we show that our scheme with constant computational complexity per dimension almost achieves the Gaussian rate distortion function, and that too universally among subgaussian sources. Moreover, the scheme and its analysis can be easily extended to sources with other tail behaviour. We believe that this approach will yield
very efficient rate-distortion codes for various sources, answering a question of fundamental interest and having many applications.}

\new{Returning to the literature on quantizers for first order stochastic optimization,
prior works
  including~\cite{alistarh2017qsgd} remain vague about the analysis for mean square bounded oracles. 
  Most of the works use gain-shape quantizers that separately quantize
  the Euclidean norm ({\em gain}) and the normalized vector ({\em
    shape}). But they operate under an engineering assumption: ``the
  standard $32$ bit precision suffices for describing the gain.'' One
  of our goals in this work is to evaluate how to efficiently use
  these 32 bits. For instance, can we use a simple uniform quantizer
  for gain? }

\new{To study such questions, we need a lower bound for gap to
  optimality for any optimization protocol using a uniform gain
  quantizer.  However, such a lower bound is not available. Indeed,
  all prior lower bounds use almost surely bounded oracles and cannot
  establish an additional limitation for mean square bounded oracles
  with heavy tails. In fact, while information theoretic lower bounds
  for SGD are well-known ($cf.$~\cite{agarwal2012information}), even
  for the almost surely bounded oracle setting, bounds for quantized
  oracles (similar to~\cite{duchi2014optimality}) have not been
  reported anywhere. We note that lower bounds for first-order
  optimization using quantized gradients are related to that of
  statistical learning and estimation when each sample must be
  quantized to a few bits
  ($cf.$~\cite{ZDJW:13,shamir2014fundamental,GMN:14,HOW:18,ACT:19,sahasranand2018extra}).
}

\new{ Also, it is interesting to compare our qualitative results with
  those in~\cite{acharya2019distributed}. The focus
  of~\cite{acharya2019distributed} was to design algorithms that
  attain the optimal convergence rate using communication that is
  sublinear in $d$. As our lower bounds show, this is impossible in
  our setting. Rather, we provide a new scheme to reduce the
  dependence of the number of bits per dimension on $T$.}

\new{Finally, the related problem of memory constrained optimization
  was stated as an open problem in \cite{woodworth2019open}. In this
  setting, we are only allowed to use a $2^{M}$ state machine to
  implement the optimization algorithm. While there is a high level
  connection between this problem and our problem, the memory
  constrained setting is more restrictive since the state of the
  algorithm must be restricted to $M$ bits at every instance, as
  opposed to our setting where only the oracle output is restricted to
  a finite precision of $r$ bits.}

 \subsection{Our contributions}\label{ss:contributions}
 \new{We start with almost surely bounded oracles and consider first
   order optimization protocols for $d$ dimensional problems with $T$
   iterations. We begin by deriving a simple information theoretic,
   precision-dependent lower bound which shows that no optimization
   protocol using a first order oracle and gradient updates of
   precision $r<d$ bits can have gap to optimality smaller than
   roughly $\sqrt{d}/\sqrt{rT}$. In particular, we need precision
   exceeding $\Omega(d)$ bits to get the classic convergence rate of
   $1/\sqrt{T}$ for convex functions.}

 \new{As our main contribution, we propose a new fixed-length
   quantizer we term {\em{Rotated Adaptive Tetra-iterated Quantizer }}
   (RATQ) that along with projected subgradient descent (PSGD) is
   merely a factor of $O(\log \log \log \log^* d)$ far from this
   minimum precision required to attain the $O(1/\sqrt{T})$
   convergence rate. In a different setting, when the precision is
   fixed upfront to $r$, we modify RATQ by roughly quantizing and sending only
a subset of coordinates of the rotated vector. 
We show that this modified version of RATQ is only a factor $O(\log
   \log^*d)$ far from the optimal convergence rate.}

\newer{For the case of mean square bounded oracles}, we establish an
information theoretic lower bound in Section \ref{s:ug} which shows
(using a heavy-tailed oracle) that the precision used for gain
quantizer must exceed $\log T$ when the gain is quantized uniformly
for $T$ iterations and we seek $O(1/\sqrt{T})$ optimization
accuracy. Thus, $32$ bits are good for roughly a billion iterations
with uniform gain quantizers, but not beyond that. Interestingly, we
present a new, adaptive gain quantizer which can attain the same
performance using only $\log \log T$ bits for quantizing gain. If one
has $32$ bits to spare for gain, then by using our quantizer we can
handle algorithms with $2^{2^{32}}$ iterations, sufficient for any
practical application.

\new{As an application of our general construction, we revisit the
  distributed mean estimation problem considered
  in~\cite{suresh2017distributed}. We show that using RATQ at each
  client requires fixed-length communication that is roughly the same
  as the optimal variable-length communication
  from~\cite{suresh2017distributed}. Furthermore, we show that RATQ
  yields a fixed-length quantizer for the unit Euclidean ball that is
  only a factor $O(\log \log \log \log^\ast d)$ from the optimal
  communication, must better in comparison to the prior known
  quantizer with $O(\log \log d)$ factor gap to optimality.}

\new{Also, we consider the Gaussian rate-distortion problem and
  evaluate the performance of a subroutine of RATQ (without
  rotation). We show that this efficient quantizer requires a
  minuscule excess rate over the classic $(1/2)\log(\sigma^2/D)$ to
  get a normalized mean square error less than $D$. Further, our
  proposed quantizer is universal and applies to any random vector
  with centered subgaussian entries.}
  

 \subsection{Remarks on techniques}\label{ss:Prior_work}
\new{ In this work we use adaptive quantizers with multiple
  dynamic-ranges $\{[-M_i, M_i]: i \in [h] \}$, with possibly a
  different dynamic range chosen for each coordinate. Once a
  dynamic-range $[-M_i, M_i]$ is chosen for a coordinate, the
  coordinate is represented using a quantized uniformly within this
  dynamic-range using $k$ levels. Using a different dynamic-range for
  each coordinate allows us to reduce error per coordinate, but costs
  us in communication since we need to communicate which $M_i$ is used
  for each coordinate. In devising our scheme, we need to carefully
  balance this tradeoff. We do this by taking recourse to the
  following observation: when the same dynamic range is chosen for all
  coordinates, the mean square error per coordinate roughly grows as
\[ 
O\left(\frac{\sum_{ i \in [h]}M_{i}^2 \cdot p(M_{i-1})}
{(k-1)^2}\right),
\] 
where $p(M)$ is the probability of the $\ell_\infty$ norm of the input
vector exceeding $M$ and $k$ denotes the number of levels of the
uniform quantizer. This observation allows us to relate the mean
square error to the tail-probabilities of the $\ell_\infty$ norm of
the input vector. In particular, we exploit it to decide on the
subvectors which we quantize using the same dynamic range.  }

\new{We use another classic trick (see~\cite{gersho2012vector}): we
  transform the input vector before we apply our adaptive
  quantizer. In particular, we use a randomized transform
that expresses the input vector over a random basis.
    The specific choice of our random transform is determined by our
  assumption for the gradients, namely that their $\ell_2$ norms are
  almost surely bounded by $B$.}

\newer{Drawing from these ideas, we propose the quantizer
  RATQ for quantizing random vectors with $\ell_2$ norm
  almost surely bounded by $B$. The main steps in RATQ are as follows:
  \begin{enumerate}
    \item {\it Rotate.} RATQ transforms the input vector by rotating
      it by multiplying with the randomized Hadamard transform which
      preservers the Euclidean norm.  This specific random transform
      was also used in \cite{ailon2006approximate} for the Fast J-L
      transform.  More recently, \cite{suresh2017distributed} used it
      to build a fixed length quantizer for distributed mean
      estimation. Incidentally, RATQ improves upon the performance of
      this fixed length quantizer for the problem of distributed mean
      estimation, as can be seen in Section
      \ref{s:distributed_mean}. In both these works, the randomized Hadamard
      transform is used to control the infinity norm of the output 
      vector.

\item {\it Adaptively quantized subvectors.} RATQ groups coordinates
  of the input vector to form smaller dimensional subvectors, after
  preprocessing the input vector using random rotation. Then, for each
  subvector the smallest dynamic range from the set $\{[-M_i, M_i] : i
  \in [h]\}$ is selected so that all the coordinates of that subvector
  lie within that range. Within this selected dynamic-range, each
  coordinate of the subvector is quantized uniformly. A key distinguishing feature of RATQ is choosing the set of $M_i$s to grow as a tetration, roughly as $M_{i+1}=e^{M_{i}}$. The large
  growth rate of a tetration allows us to cover the complete range of
  each coordinate using only a
  small number of dynamic ranges, which leads to an unbiased quantizer and reduces the
  communication. Also, after random rotation, each
  coordinate of the vector is a centered subgaussian random
  variable with a variance-parameter of $O(B^2/d)$, which, despite the large
  growth rate of a tetration, ensures that the per coordinate mean
  square error between the quantized output and the input is almost a
  constant.
\end{enumerate}
}

\newer{We remark that using an adaptively chosen dynamic-range can
  alternatively be implemented by transforming the input using a
  monotone function. This, too, is a classic technique in quantization
  known as {\it companding}
  ($cf.$~\cite{gersho2012vector}). Companding is known as a popular
  alternative to entropic coding for fixed-length codes. However, to
  the best of our knowledge, our paper is the first to combine it with
  other techniques and rigorously analyze it for the $\ell_2$ norm
  bounded vector quantization problem. Perhaps it is a bit surprising
  that this combination of classic technique was not analysed for
  constructing an efficient covering of the unit Euclidean ball, the
  problem underlying our quantization problem. We separately highlight the performance of RATQ as a covering for the unit Euclidean ball in Section~\ref{s:distributed_mean}, in the context of distributed mean estimation.}

\newer{Moving to oracles with mean square
  bounded $\ell_2$ norms, we take recourse to gain-shape quantizers
  and quantize the (normalized) shape vector using RATQ. However,
  unlike prior work, we rigorously treat gain quantization. Our
  proposed quantizer for gain is once again an adaptive uniform
  quantizer, but this time we cannot use a tetration for selecting
  possible dynamic-ranges $M_i$s since gain need not be
  subgaussian. We now only have tail-probability bounds determined by
  the Markov inequality (heavy-tails) and can only increase $M_i$s
  geometrically.}

\newer{In fact, the choice of $M_i$s for both the gain-quantizer and RATQ above
  is based on our general procedure for selecting $M_i$s based on the
  tail-probability bounds for the coordinates.  As another
  instantiation of this principle, we study the Gaussian
  rate-distortion problem where we have a handle over these
  tail-probabilities, even without any additional transforms applied
  to the input random vector. 
}

\newer{To extend our quantizers (for both almost surely and mean square bounded oracles)
  to the fixed precision setting where only $r$ less than $d$ bits can
  be sent, we take recourse to the standard uniform subsampling
  technique. Specifically, we uniformly 
sample  $O(r)$ coordinates from $[d]$ (without replacement)  and
communicate quantized values only for these coordinates.} 

\new{Our lower bounds draw from an oracle complexity lower bound derived
  in~\cite{agarwal2009information} and use a strong data processing
  inequality from~\cite{duchi2014optimality}. Similar ideas have
  appeared in lower bounds for communication constrained statistics;
  see, for instance,~\cite{ZDJW:13, BGMNW:16,XR:18, ACT:18}. However, this only
allows us to obtain lower bounds for the almost surely bounded setting. For the mean square bounded setting, we need a 
new construction with  ``heavy tails''. In particular, our proposed heavy-tailed construction shows a bottleneck for uniform gain quantizers which can be circumvented by our proposed quantizer, thereby establishing a strict improvement over uniform gain quantizers.}

\newer{We remark that independent of our work, an adaptive quantizer
similar to the one with we use for gain-quantization with geometrically increasing $M_i$s  appears
in~\cite{ramezani2019nuqsgd}.
Note that we use this quantizer for
gain-quantization, while \cite{ramezani2019nuqsgd} uses it to quantize
the shape.
However, the setting considered is
that of~\cite{alistarh2017qsgd} where quantization is followed by entropic compression. In particular, the fixed-length 
performance is suboptimal and mean square bounded oracles are not
handled in the worst-case. Another recent independent work~\cite{gandikota2019vqsgd}
presents a different scheme where a different random transform is used
instead of random rotation.
However, the
goal of this work is different from ours, and in particular, it has
much worse communication requirement compared to our scheme. 
}

 \subsection{Organization}
 \new{We formalize our problem in the next section and describe our results
   for almost surely and mean
 square bounded oracles in Sections~\ref{s:as} and~\ref{s:ms},
 respectively, along with some of the shorter proofs. The more
 elaborate proofs are provided in 
 Section~\ref{s:proof}, with additional details relegated to the
 appendix.
 We present the application of our quantizers to the problem 
 distributed mean estimation in Section~\ref{s:distributed_mean}
 and Gaussian rate-distortion in Section~\ref{s:gaussian_rate}. 
 }

%% file: prelim.tex
\section{The setup and preliminaries}
\subsection{Problem setup}\label{s:problemsetup}
We fix the number of
iterations $T$ of the optimization algorithm (the number of times the
first order oracle is accessed) and the precision $r$ allowed to
describe each subgradient. Our fundamental metric of performance is the minimum error 
 (as a function of $T$ and $r$) with which such an algorithm can find the optimum value.

Formally, we want to find the minimum value of an unknown convex function $\displaystyle{f:\X
  \rightarrow \R}$ using {\em oracle access} 
  to noisy subgradients of the
  function ($cf.$~\cite{nemirovsky1983problem,bubeck2015convex}).
We assume that the function $f$ is convex over the compact, convex
domain $\X$ such that $\sup_{x,y \in \X}\norm{x-y}_2 \leq
D$; we denote the set of all such $\X$ by $\mathbb{X}$.
For a query point $x\in \X$, the oracle outputs random 
estimates of the subgradient $\hat{g}(x)$ which 
 for all  $x \in \X$ satisfy
\begin{equation}\label{e:asmp_unbiasedness}
\E{\hat{g}(x)|x} \in \partial f(x),
\end{equation}
\begin{equation}\label{e:asmp_L2_bound}
\E{\norm{\hat{g}(x)}_2^2|x} \leq B^2,
\end{equation}
where $\partial f(x)$ denotes the set of subgradients of $f$ at $x$.
\begin{defn}[Mean square bounded oracle]
A first order oracle which upon a query $x$ outputs the
subgradient estimate $\hat{g}(x)$ satisfying the assumptions
\eqref{e:asmp_unbiasedness} and \eqref{e:asmp_L2_bound} is termed a
mean square bounded oracle. We denote by $\oO$ the set of pairs $(f,O)$ with a convex function $f$ and a mean square bounded oracle $O$. 
\end{defn}


The variant with {almost surely}  bounded
oracles has also been considered 
($cf.$~\cite{nemirovsky1983problem,agarwal2012information}),   where we assume for all $x \in \X$
\begin{equation}\label{e:asmp_as_bound}
P(\norm{\hat{g}(x)}_2^2 \leq B^2|x)=1.
\end{equation}
 
\begin{defn}[Almost surely bounded oracle]
A first order oracle which upon a query $x$   outputs only the subgradient estimate $\hat{g}(x)$ satisfying the assumptions \eqref{e:asmp_unbiasedness} and \eqref{e:asmp_as_bound} is termed an almost surely bounded oracle. We denote the class of convex functions and 
oracle's satisfying 
assumptions \eqref{e:asmp_unbiasedness} and 
\eqref{e:asmp_as_bound} by $\oO_0$.
\end{defn}
In our setting, the outputs of the oracle are passed through a quantizer. An {\em $r$-bit quantizer} consists of randomized
mappings $(\Qenc, \Qdec)$ with the encoder mapping
$\Qenc:\R^d\to\{0,1\}^{r}$ and the decoder mapping $\Qdec: \{0,1\}^r\to \R^d$. The
overall quantizer is given by the composition mapping $Q=\Qdec\circ \Qenc$. Denote by 
$\Q_r$ the set of all such $r$-bit quantizers.   

For an oracle $(f, O)\in \oO$ and an $r$-bit quantizer $Q$,
let $QO= Q\circ O$ denote the 
composition oracle that outputs 
$Q(\hat{g}(x))$ for each query $x$. Let $\pi$ be an
algorithm with at most $T$ iterations with oracle access to
$QO$. We will call such an algorithm an {\em optimization  protocol}. 
Denote by $\Pi_T$ the set of all such optimization protocols with $T$
iterations.

Denoting the combined optimization protocol with its oracle $QO$ by
$\pi^{QO}$ and the 
associated output as $x^*(\pi^{QO})$, we measure the performance of
such an optimization protocol for a given $(f,O)$ using the metric
$\ep(f, \pi^{QO})$ defined as  $\ep(f, \pi^{QO}) \eqdef \E{f(x^*(\pi^{QO}))-\min_{x\in \X} f(x)}$.
The fundamental quantity of interest in this work are minmax errors
\begin{align*}
  \mathcal{E}_0^*{(T,r)} &\eqdef \sup_{\X \in \mathbb{X}} \inf_{\pi \in \Pi_T}\inf_{Q \in
    \mathcal{Q}_r}\sup_{(f, O) \in \oO_0}\ep(f,\pi^{QO}),
  \\
   \ep^*{(T,r)} &\eqdef
  \sup_{\X \in \mathbb{X}}\inf_{\pi \in \Pi_T}\inf_{Q \in
    \mathcal{Q}_r}\sup_{(f, O) \in \oO}\ep(f, \pi^{QO}).
 \end{align*}
Clearly, $\mathcal{E}^*{(T,r)} \geq\mathcal{E}_0^*{(T,r)}$.

\newest{\begin{rem}
We restrict to {\em memoryless} quantization schemes where the same quantizer will be applied to each new gradient vector, without using any information from the previous updates. 
Specifically, at each instant $t$ and for any precision $r$, the quantizers in  $\Q_r$ do not use any information from the previous time instants to quantize the subgradient outputted by $O$ at $t$. 
\end{rem}
}

 \subsection{A benchmark from prior results}\label{ss:benchmark}
 We recall results for the classic setting with $r=\infty$. Prior work
gives a complete characterization of the minmax errors
$\ep_0^*{(T,\infty)}$ and $\ep^*{(T,\infty)}$ for this setting; see,
for instance, \cite{nemirovsky1983problem,
  nemirovski1995information,agarwal2009information}. 
We summarize these well-known results below
($cf.$~\cite{nemirovsky1983problem}, \cite[Theorem 1a]{agarwal2012information}).
\begin{thm}\label{t:e_infty}
For an absolute constant $c_0$, we have
\[\frac{DB}{ \sqrt{T} } \geq \ep^*(T,\infty) \geq \ep^*_0(T,\infty) \geq \frac{c_0DB}{ \sqrt{T} }  .\]
\end{thm}

\subsection{Quantizer performance for finite precision optimization}
Our overall optimization protocol throughout is the {\em projected SGD} (PSGD)
(see~\cite{bubeck2015convex}).  In fact, we establish lower bound
showing roughly the optimality of PSGD with our quantizers. 

In PSGD the  standard SGD updates are projected back 
 to the domain using the projection map $\Gamma_\X$ given by 
$\Gamma_{\X}(y) := \min_{x \in \X} \norm{x-y}_2.$
We use the {\em quantized PSGD} algorithm described in Algorithm~\ref{a:SGD_Q}. 
\begin{figure}[h]
\centering
\begin{tikzpicture}[scale=1, every node/.style={scale=1}]
\node[draw,text width= 8 cm , text height= ,] {%
\begin{varwidth}{\linewidth}       
            \algrenewcommand\algorithmicindent{0.7em}
\begin{algorithmic}[1]
   \Statex \textbf{Require:} $x_0\in \X, \eta \in \R^+$, $T$ and
   access to composed oracle $QO$ \For{$t=0$ to $T-1$}

$x_{t+1}=\Gamma_{\X} \left(x_{t}-\eta Q(\hat{g}(x_{t}))\right)$
   \EndFor \State \textbf{Output:} $\frac 1 T \cdot {\sum_{t=1}^T x_t}$
\end{algorithmic}
\end{varwidth}};
 \end{tikzpicture}

 \renewcommand{\figurename}{Algorithm}
\caption{Quantized PSGD with quantizer $Q$}
\label{a:SGD_Q}

\end{figure}

The quantized output $Q(\hat{g}(x_t))$, too, constitutes a noisy
oracle, but it can be biased for mean square bounded oracles. Though biased first-order
oracles were considered in~\cite{hu2016bandit}, the effect of
quantizer-bias has not been studied in the past.
The performance of a quantizer $Q$, 
when it is used with PSGD  for mean square  bounded oracles, is controlled by 
the worst-case $L_2$ norm $\alpha(Q)$ of its output and the worst-case
bias $\beta(Q)$ defined as\footnote{We omit the
dependence on $B$ and $d$ from our notation.}
\begin{align}
    \alpha(Q)&\eqdef \sup_{Y \in \R^d: \E{\norm{Y}_2^2}\leq B^2}
    \sqrt{\E{\norm{Q(Y)}_2^2}},
    \nonumber
\\
\beta(Q)&\eqdef \sup_{Y \in \R^d:
  \E{\norm{Y}_2^2}\leq B^2} \norm{\E{Y-Q(Y)}}_2.
\label{e:alpha,beta}
\end{align}
The corresponding quantities for almost surely bounded oracles are
    \begin{align}
      \alpha_0(Q)&\eqdef \sup_{Y \in \R^d: \norm{Y}_2\leq B \text{ a.s.}}
      \sqrt{\E{\norm{Q(Y)}_2^2}},
      \nonumber
  \\
  \beta_0(Q)&\eqdef \sup_{Y \in \R^d:
    \norm{Y}_2\leq B \text{ a.s.}} \norm{\E{Y-Q(Y)}}_2.
  \label{e:alpha0,beta0}
    \end{align}
    Using a slight modification
of the standard proof of convergence for PSGD, we get the following result.
\begin{thm}\label{t:basic_convergence}
For any quantizer $Q$, the output $x_T$ of optimization protocol $\pi$
given in Algorithm \ref{a:SGD_Q} satisfies
\begin{align*}
\sup_{(f, O) \in \oO_0}\ep(f, \pi^{QO})&\leq
D\left(\frac{\alpha_0(Q)}{\sqrt{T}}+ \beta_0(Q)\right),
\\
  \sup_{(f, O) \in \oO}\ep(f, \pi^{QO})&\leq
D\left(\frac{\alpha(Q)}{\sqrt{T}}+ \beta(Q)\right),
\end{align*}
when the parameter $\eta$ is set to $D/(\alpha_0(Q) \sqrt{T})$ and
$D/(\alpha(Q) \sqrt{T})$, respectively.
\end{thm}
\noindent See Appendix \ref{ap:QPSGD} for the proof.
\begin{rem}[Choice of learning rate]  We fix  the parameter $\eta$ of Algorithm \ref{a:SGD_Q} to $D/(\alpha_0(Q) \sqrt{T})$  and $D/(\alpha(Q) \sqrt{T})$ for all the results in Section \ref{s:as} and Section \ref{s:ms}, respectively.
  \end{rem}

%% file: almostsure.tex
\section{Main results for almost surely bounded oracles}\label{s:as}
Our main results will be organized along two regimes: the high-precision
and the low-precision regime. For the 
high-precision regime, we seek to attain the optimal convergence rate of
$1/\sqrt{T}$ using the minimum precision possible. For the low-precision regime, we 
seek to attain the fastest convergence rate possible for a given, fixed precision $r$.

\subsection{A precision-dependent lower bound}\label{ss:simple_lb}
We begin with a simple refinement of the lower bound
implied by Theorem~\ref{t:e_infty}
The proof of this result is obtained by
appropriately modifying the proof in \cite{agarwal2012information},
along with the strong data processing inequality in
\cite{duchi2014optimality}.
\begin{thm}\label{t:e_r_LB}
There exists an absolute constant $c$, independent of $d$, $T$, and $r$ such that 
\[\mathcal{E}^*{(T, r)} \geq\mathcal{E}_{0}^*{(T, r)} \geq 
 \frac{cDB}{\sqrt{T}} \cdot \sqrt{\frac{d}{\min\{d,r\}}} .\]
\end{thm}
\begin{proof}

The proof of lower bound in Theorem~\ref{t:e_r_LB} is a slight extension of the standard proof of Theorem~\ref{t:e_infty}. 
We provide a sketch for completeness. For simplicity, we assume $\X=\{x:\norm{x}_\infty\leq D/(2\sqrt{d}) \}$.
Let $\V\subset\{-1,1\}^d$ be the maximal $d/4$-packing in Hamming distance, namely it is a collection of vectors such that
any two vectors $\alpha, \alpha^\prime\in V$, $d_H(\alpha, \alpha^\prime)\geq d/4$. As is well-known, there exists such a packing of cardinality $2^{c_2 d}$, where $c_2$ is a constant.
Consider convex functions $f_\alpha$, $\alpha\in \V$, with domain $\X$ and satisfying assumptions \eqref{e:asmp_unbiasedness} and
\eqref{e:asmp_as_bound} given below:
 \[
f_{\alpha}(x):=\frac {B\delta}{\sqrt{d}} \sum_{i=1}^{d} \alpha(i) x(i).
\]
Note that the gradient of $f_\alpha(x)$ is given by $B\alpha/\sqrt{d}$ for each $x\in \X$. 
For each $f_\alpha$, consider the corresponding gradient oracles $O_\alpha$ which outputs 
independent values for each coordinate,
with the value of $i$th coordinate taking values $B/\sqrt{d}$ and $-B/\sqrt{d}$ with probabilities $(1+2\delta\alpha(i))/2$ and $(1-2\delta\alpha(i))/2$, respectively. We denote the distribution of output of oracle $O_\alpha$ by $P_\alpha$.

Let $V$ be distributed uniformly over $\V$. Consider the multiple hypothesis testing problem of determining $V$ by observing
samples from $\Qenc(Y)$ with $Y$ distributed as $P_V$. Consider an optimization algorithm that outputs $x_T$ after $T$ iterations. Then, we have 

\begin{align*}
\E{f_\alpha(x_T)-f_\alpha(x^*)}&\geq \frac{DB\delta}{8} P\left(f_\alpha(x_T)-f_\alpha(x^*)\geq\frac{DB\delta}8\right)
\\
&{=} \frac{DB\delta}{8} P\left( \frac {B\delta}{\sqrt{d}} \alpha^T(x_T-x^*) \geq\frac{DB\delta}8 \right)
\\
&{=} \frac{DB\delta}{8} P\left( \frac {B\delta}{\sqrt{d}} \norm{x_T-x^*}_{1} \geq\frac{DB\delta}8 \right)
\\
&{=}  \frac{DB\delta}{8} P\left(\norm{(2\sqrt{d}/D)x_T + \alpha}_1\geq \frac {d}{4}\right),
\end{align*}
where the second identity holds since
  $sign(\alpha(i))=sign(x_T-x^*)$ and the final identity is obtained by noting that the optimal value $x^*$ for $f_\alpha$ is $-(D/2\sqrt{d})\alpha$.  Note that all $\alpha,
\alpha^\prime \in\V$ satisfy $\norm{\alpha-\alpha^\prime}_1\geq d/2
$. Consider the following test for the aforementioned hypothesis
testing problem. We execute the optimization protocol using oracle
$O_V$ and declare the unique $\alpha\in V$ such that
$\norm{(2\sqrt{d}/D)x_T + \alpha}_1 < d/4$. The probability of error
for this test is bounded above by $P\left(\norm{(2\sqrt{d}/D)x_T +
  \alpha}_1\geq \frac{d}{4}\right)$, whereby the previous bound and
Fano's inequality give
\[
\E{f_\alpha(x_T)-f_\alpha(x^*)}\geq \frac{DB\delta}{8} \left(1-
\frac{TI(V\wedge Q(Y))+1}{\log |\V|}\right).
\]
For a quantizer $Q$ with precision $r$, using the strong data
processing inequality bound from~\cite[Proposition
  2]{duchi2014optimality}, we have $I(V\wedge Q(Y))\leq 360
\delta^2\min\{r, d\}$. Therefore,
\[
\max_{\alpha}\ep_0(f, \pi^{QO})\geq \frac{{DB}\delta}{8} \bigg(1-
\frac 1{c_2d} - \frac{360T\delta^2\min\{r, d\}}{c_2d}\bigg).
\]
The proof is completed by maximizing the right-side over $\delta$.
\end{proof}


As a corollary, we get that there is no hope of
getting the desired convergence rate of $1/\sqrt{T}$ by using a
precision of less than $d$.  
\begin{cor}\label{c:OmegaD}
For $\mathcal{E}_0^*{(T, r)}$  or $\mathcal{E}^*{(T, r)}$ to be less than ${DB}/{\sqrt{T}}$, the
precision $r$ must be at least $\Omega(d)$.
\end{cor}
\subsection{RATQ: Our quantizer for the  $\ell_2$ ball}\label{s:alg_RATQ}
We propose {\em{Rotated Adaptive
    Tetra-iterated Quantizer }}(RATQ) to quantize any random vector
$Y$ with $\norm{Y}_2^2\leq B^2$, which is what we need for almost surely bounded oracles.
RATQ first rotates the input vector, then divides the coordinates of the rotated vectors into smaller groups,
and finally quantizes each subgroup-vector using a {\em Coordinate-wise Uniform Quantizer} (CUQ). 
However, the dynamic-range used for each subvector is chosen adaptively from a set of tetra-iterated levels. 
We call this adaptive quantizer {\em Adaptive Tetra-iterated Uniform Quantizer} (ATUQ), and it is  the main workhorse of our construction. The encoder and decoder for RATQ are given in Algorithm~\ref{a:E_RATQ} and Algorithm~\ref{a:D_RATQ}, respectively. 
The details of all the components involved are described below.
 
 \begin{figure}[ht]
\centering
\begin{tikzpicture}[scale=1, every node/.style={scale=1}]
\node[draw, text width= 9 cm, text height=,] {%
\begin{varwidth}{\linewidth}
            
            \algrenewcommand\algorithmicindent{0.7em}
 \renewcommand{\thealgorithm}{}
\begin{algorithmic}[1]
\Require  Input $Y\in \R^d$, rotation matrix R
  
 \State Compute $\tilde{Y}=RY$ 

\For {$i \in
    [d/s]$}
  
  $\displaystyle{\tilde{Y}_i ^{T} =[\tilde{Y}((i-1)s+1), \cdots \tilde{Y}(\min\{is,d\})]^{T}}$ 
 
  \EndFor

   \State \textbf{Output:} 
   $\Qenc_{{\tt at}, R}(Y) =\{\Qenc_{{\tt at}}(\tilde{Y_1})\cdots\Qenc_{{\tt
      at}}(\tilde{Y}_{\ceil{d/s}})\}$
\end{algorithmic}
\end{varwidth}};
 \end{tikzpicture}
 \renewcommand{\figurename}{Algorithm}
 \caption{Encoder $\Qenc_{{\tt at}, R}(Y)$ for RATQ}\label{a:E_RATQ}
 \end{figure}

 \begin{figure}[ht]
\centering
\begin{tikzpicture}[scale=1, every node/.style={scale=1}]
\node[draw, text width= 9 cm, text height=,] {%
\begin{varwidth}{\linewidth}
            \algrenewcommand\algorithmicindent{0.7em}
 \renewcommand{\thealgorithm}{}
\renewcommand{\thealgorithm}{}
\begin{algorithmic}[1]
    \Require 
  Input $\{Z_i, j_i \}$ for $i \in [\ceil{d/s}]$, rotation matrix R

\State $Y^{T}= [\Qdec_{\tt at}( Z_1, j_1 ), \cdots$,
  $\Qdec_{{\tt at}}( Z_{\ceil{d/s}}, j_{\ceil{d/s}} ) ]^T$
\State \textbf{Output:} $\Qdec_{{\tt at},R}( \{Z_i, j_i\}_{i=1}^{\ceil{d/s}})=R^{-1} Y$
\end{algorithmic}
\end{varwidth}};
 \end{tikzpicture}
 \renewcommand{\figurename}{Algorithm}
 \caption{Decoder $\Qdec_{{\tt at}, R}(Z, j)$ for RATQ}\label{a:D_RATQ}
 \end{figure}

\paragraph{Rotation and division into subvectors.} 
RATQ first rotates the input vector by multiplying it with a random Hadamard matrix.
Specifically, denoting by $H$ the $d\times d$
Walsh-Hadamard Matrix (see \cite{horadam2012hadamard})\footnote{We
  assume that $d$ is a power of $2$.}, define
\begin{equation}\label{e:R}
R\eqdef \frac{1}{\sqrt{d}}\cdot HD,
\end{equation}
where $D$ is a diagonal matrix with each diagonal entry generated uniformly from $\{-1, +1\}$. 
The input vector $y$ is multiplied by $R$ in the rotation step. 
The matrix $D$ can be generated using shared randomness
between the encoder and decoder.

Next, the rotated vector of dimension $d$ is partitioned into
 $\ceil{d/s}$ smaller subvectors. The $i^{th}$ subvector comprises
the coordinates $\{(i-1)s+1,\cdots, \min\{is,d\}\}$,
 for all $i \in [d/s].$ Note that the dimension of all the sub
 vectors except the last one is $s$, with the last one having a
 dimension of $d-s\floor{d/s}$.

\paragraph{Coordinate-wise Uniform Quantizer (CUQ).} RATQ uses CUQ as a subroutine;  we describe the latter for $d$ dimensional inputs, but it will only be applied to subvectors of lower dimension in RATQ. CUQ has a dynamic range $[-M, M]$ associated with it, and it uniformly quantizes each coordinate of the input to $k$-levels as long as the component is
within the dynamic-range $[-M, M]$. 
Specifically, it 
partitions the interval $[-M, M]$ into parts $I_\ell\eqdef
(B_{M,k}(\ell), B_{M,k}(\ell+1)]$, $\ell\in\{0,\ldots, k-1\} $, where
  $B_{M,k}(\ell)$ are given by \eq{ &B_{M,k}(\ell) := -M+\ell \cdot
    \frac{2M}{k-1}, \quad \forall\,  \ell \in \{0,\ldots,
    k-1\}.} 
Note that  we need to communicate $k+1$ symbols per coordinate -- $k$ of these symbols correspond to the $k$ uniform levels and the additional   symbol corresponds to the overflow symbol $\emptyset$.  Thus we need a total precision of $d\ceil{\log(k+1)}$ bits to represent the output of the CUQ encoder. The encoder and decoders used in CUQ are given in Algorithms~\ref{a:E_CUQ} and~\ref{a:D_CUQ}, respectively. 
In the decoder, we have set $B_{M,k}(\emptyset)$ to $0$.

\begin{figure}[ht]
\centering
\begin{tikzpicture}[scale=1, every node/.style={scale=1}]
\node[draw,text width= 9 cm, text height=,] {%
\begin{varwidth}{\linewidth}
            
            \algrenewcommand\algorithmicindent{0.7em}
\begin{algorithmic}[1]
\vspace{-0.5cm}
\Require Parameters  $M \in \R^+$ and input $Y \in \R^d$
  \For {$i \in [d]$} \If{$|Y(i)| > M$}
   
   $Z(i)=\emptyset$
 
 \Else \For {$\ell \in \{0,\ldots, k-1\}$}\label{step:UQ} \If {${Y}(i)
   \in (B_{M,k}(\ell), B_{M,k+1}(\ell+1)]$}
 \Statex  \hspace{1cm}
     $
      {Z}(i) =
\begin{cases}
\ell+1, \quad w.p. ~ \frac{{Y}(i) - B_{M,k}(\ell)}{B_{M,k}(\ell+1)-B_{M,k}(\ell)}
 \\ \ell, \hspace{0.6cm} \quad w.p. ~  
  \frac{B_{M,k}(\ell+1) - {Y}(i)}{B_{M,k}(\ell+1)-B_{M,k}(\ell)} 
\end{cases}
$
   
\EndIf \EndFor \EndIf \EndFor \State \textbf{Output:} $\Qenc_{\tt u}(Y; M)=Z$
\end{algorithmic}  
\end{varwidth}};
 \end{tikzpicture}
 \renewcommand{\figurename}{Algorithm}
 \caption{Encoder  $\Qenc_{\tt u}(Y; M)$ of CUQ}\label{a:E_CUQ}
 \end{figure}

 \begin{figure}[ht]
\centering
\begin{tikzpicture}[scale=1, every node/.style={scale=1}]
\node[draw, text width= 9 cm, text height=,] {%
\begin{varwidth}{\linewidth}
            
            \algrenewcommand\algorithmicindent{0.7em}
 \renewcommand{\thealgorithm}{}
\begin{algorithmic}[1]
  \Require  Parameters  $M \in \R^+$and input $Z \in \{0,
  \dots, k-1, \emptyset\}^d$ \State Set $\hat{Y}(i)= B_{M,k}(Z(i))$, for all $i\in
       [d]$ \State \textbf{Output:} $\Qdec_{\tt u}(Z;
       M)=\hat{Y}$\label{step:output_coordinate}
\end{algorithmic}
\end{varwidth}};
 \end{tikzpicture}
 \renewcommand{\figurename}{Algorithm}
 \caption{Decoder $\Qdec_{\tt u}(Z; M)$ of CUQ}\label{a:D_CUQ}
 \end{figure}


\paragraph{Adaptive Tetra-iterated Uniform Quantizer (ATUQ).}
The quantizer ATUQ is CUQ with its dynamic-range chosen in an adaptive manner. 
In order to a quantize a particular input vector, it first chooses a 
 dynamic range from $[-M_i, M_i]$, $1\leq i\leq h$.
 To describe these $M_i$s, we first define the $i^{th}$
tetra-iteration for $e$, denoted by {$e^{*i}$, recursively as follows:
\[
e^{*1}:=e, \quad e^{*i}:=e^{e^{*(i-1)}}, \quad i\in \N.
 \]
Also, for any non negative number $b$, we define 
$\ln^*b :=\inf\{i \in \N :  e^{*i} \geq b \}.$}
\new{With this notation, the values $M_i$s are defined in terms of 
$m$ and $m_0$ as follows: \eq{  
M_{0}^2= m+m_0, \quad  M_{i}^2= m \cdot e^{*i}+m_0,\quad \forall\, i \in
  \{1, \ldots, h-1\}.}}
ATUQ finds the smallest level $M_i$ which
 bounds the infinity norm of the input vector; if no such $M_i$ exists, it simply uses $M_{h-1}$.  
It then uses CUQ with dynamic range $[-M_i, M_i]$ to quantize the input vector. In RATQ, we apply ATUQ to each subvector.
The decoder of ATUQ is simply the decoder of CUQ using the dynamic range outputted by the ATUQ encoder.

   Note that in order to represent the output of ATUQ for $d$ dimensional inputs, we need a precision of  at the most $\ceil{\log h}+d\ceil{\log (k+1)}$ bits: $\ceil{\log h}$ bits to represent the dynamic range and at the most $d\ceil{\log (k+1)}$ bits to represent the output of CUQ. 
The encoder and decoder for ATUQ are given in Algorithms~\ref{a:E_ATQ} and~\ref{a:D_ATQ}, respectively.

\begin{figure}[ht]
\centering
\begin{tikzpicture}[scale=1, every node/.style={scale=1}]
\node[draw, text width= 9 cm, text height=,] {%
\begin{varwidth}{\linewidth}
            
            \algrenewcommand\algorithmicindent{0.7em}
 \renewcommand{\thealgorithm}{}
\begin{algorithmic}[1]
\Require Input $Y \in \R^d$ \If { $\norm{Y}_\infty > M_{h-1}$}
 
 Set $M^*=M_{h-1}$

\Else

Set $j^*
 =\min \{
 j: \norm{Y}_{\infty} \leq M_j\},$
  $M^*=M_{j^*}$ \EndIf \State Set $Z=\Qenc_{\tt u}(Y; M^*, k)$

\State \textbf{Output:} $\Qenc_{\tt at}(Y)=\{Z, j^*\}$
\end{algorithmic}
\end{varwidth}};
 \end{tikzpicture}
 \renewcommand{\figurename}{Algorithm}
 \caption{Encoder $\Qenc_{{\tt at}}(Y)$ for ATUQ}\label{a:E_ATQ}
 \end{figure}

\begin{figure}[ht]
\centering
\begin{tikzpicture}[scale=1, every node/.style={scale=1}]
\node[draw, text width= 9 cm, text height=,] {%
\begin{varwidth}{\linewidth}
            \algrenewcommand\algorithmicindent{0.7em}
 \renewcommand{\thealgorithm}{}
\renewcommand{\thealgorithm}{}
\begin{algorithmic}[1]
  \Require
  Input $\{Z,j\}$ with $Z\in \{0,\dots, k-1, \emptyset \}^d$ and $j\in \{0, \dots
  h-1\}$
 
  \State \textbf{Output:} $\Qdec_{\tt at}(Z,j)= \Qdec_{\tt u}(Z;
  M_{j})$
 
\end{algorithmic}
\end{varwidth}};
 \end{tikzpicture}
 \renewcommand{\figurename}{Algorithm}
 \caption{Decoder $\Qdec_{\tt at}(Z, j)$ for ATUQ}\label{a:D_ATQ}
 \end{figure}

When ATUQ is applied to each subvector in RATQ, each of the  $\ceil{d/s}$ subvectors are represented using less than $\ceil{\log h} + s\ceil{\log (k+1)}$ bits. Thus, the overall precision for RATQ is less than\footnote{$\log$ denotes the logarithm to the base $2$,  $\ln$ denotes logarithm to the base $e$.} 
\[\ceil{d/s}\cdot \ceil{\log h}+ d \ceil{\log (k+1)}\] bits.
The decoder of RATQ is simply formed  by collecting the output of the ATUQ decoders for all the subvectors to form a $d$-dimensional vector, and rotating it back using the matrix $R^{-1}$  (the inverse of the rotation matrix used at the encoder).

                             {
\paragraph{Choice of parameters.}
\new{Throughout the remainder of this section, we set our parameters $m$, $m_0$, and $h$ as follows
\begin{align}
m=\frac{3B^2}{d}, \quad m_0=\frac{2B^2 }{d} \cdot  \ln s, \quad \log h=\ceil{\log(1+\ln^\ast(d/3))}.
\label{e:RATQ_levels}
\end{align}}
In particular, this results in $M_{h-1} \geq B$ whereby,
for an input $Y$ with $\norm{Y}_2^2 \leq B^2$, RATQ outputs an unbiased estimate of $Y$. 
}



\subsection{RATQ in the high-precision regime}\label{s:RATQh}
The following result shows that RATQ is unbiased for almost surely bounded inputs and provides a bound for its worst-case second order moment; this constitutes a key technical tool for characterizing the performance of RATQ.
\begin{thm}[Performance of RATQ]\label{t:RATQ_alpha0_beta0}
Let $Q_{{\tt at}, R}$ be
the quantizer RATQ with $M_j$s set by~\eqref{e:RATQ_levels}. Then, 
for all $s,k\in \N$,
\new{\begin{align} \label{e:RATQ_alpha_bound}
\alpha_0(Q_{{\tt at}, R})&\leq B\sqrt{
\frac{9+3\ln s}{(k-1)^2}+1}, \quad 
\beta_0(Q_{{\tt at}, R})= 0. 
\end{align}}
\end{thm}
\new{\noindent The proof is deferred to Section~\ref{s:ProofRATQ}.}

Thus, $\alpha_0$ is lower when $s$ is small, but the overall precision needed
grows since the number of subvectors increases. The following choice of
parameters yields almost optimal performance:
\begin{align}
  s=\log h,\quad \log(k+1) = \ceil{\log (2  + \sqrt{9  + 3\ln s })}. 
\label{e:RATQ_bits}
\end{align}
For these choices, we obtain the following.
\begin{cor}\label{c:PSGD_RATQ_0} The overall precision $r$ used by the quantizer $Q=Q_{{\tt at},R}$
  with parameters set as in~\eqref{e:RATQ_levels},~\eqref{e:RATQ_bits} satisfies
\[
r\leq d(1+\Delta_1)+\Delta_2,
\] 
where
\new{$\Delta_1=\ceil{\log \left(2+  \sqrt{  9 + 3 \ln \Delta_2 }\right)}$ and 
$\Delta_2=\ceil{\log (1+\ln^*({d}/{3})) }$.}\\
Furthermore, the optimization protocol $\pi$
  given in Algorithm~\ref{a:SGD_Q} satisfies 
  \[\sup_{(f,O)\in \oO_0}\ep(f, \pi^{QO}) \leq \frac{\sqrt{2}DB}{\sqrt{T}}.\]
\end{cor}
\new{
\begin{proof}
By the description RATQ, it encodes the subgradients using a fixed-length code of at the most 
$ \ceil{d/s}\cdot \ceil{\log h}+ d \ceil{\log (k+1)}$ bits. Upon substituting $s$, $\log h$, and $\log(k+1)$
as in \eqref{e:RATQ_bits} and \eqref{e:RATQ_levels}, we obtain that
the total precision is bounded above by $d(1+\Delta_1)+\Delta_2$. 

For the second statement of the corollary, we have
\eq{
\sup_{(f, O) \in \oO_0}\ep(f, \pi^{QO}) &
\leq
D\left(\frac{\alpha_0(Q_{{\tt at}, R})}{\sqrt{ T}}+ \beta_0(Q_{{\tt at}, R})\right)\\
&\leq \frac{DB}{\sqrt{ T}}\cdot\sqrt{ \frac{9+3 \ln s}{(k-1)^2}+1 }\hspace{0.1cm}\\
&\leq \frac{\sqrt{2}DB }{\sqrt{T}},
}
where the first inequality follows by Theorem \ref{t:basic_convergence}, the second inequality follows by upper bounding $\alpha_0(Q_{{\tt at}, R})$ and $\beta_0(Q_{{\tt at}, R})$ using Theorem \ref{t:RATQ_alpha0_beta0}, and the third follows by substituting the parameters in the corollary statement.
\end{proof}
}

{
\begin{rem}
The precision requirement in Corollary \ref{c:PSGD_RATQ_0} matches the $d$ bit lower bound of Corollary \ref{c:OmegaD} upto a multiplicative factor of $O\left(\log \log \log \ln^*(d/3) \right).$
\end{rem}
}

\subsection{Comparison with Quantized Stochastic Gradient Descent
  (QSGD)} \label{s:QSGD}
At this point, it will be instructive to compare our results with a
state of the art scheme from~\cite{alistarh2017qsgd} -- Quantized Stochastic
Gradient Descent(QSGD). We have the following result as a
  consequence of  \cite[Lemma 3.1, Corollary 3.3]{alistarh2017qsgd}.

\begin{thm}
Under the assumption that the $\ell_2 $norm of the subgradient estimate
can be communicated using $F$ bits of communication, QSGD achieves for
any $(f, O)\in \oO_0$
  \[
  \ep(f, \pi^{QO}) \leq \frac{\sqrt{2}DB}{\sqrt{T}},
  \]
  using a variable-length code of expected precision at every
  iteration less than $F+2.8d$.
\end{thm} 

To compare this performance with that of our proposed RATQ, we will show that unless the dimension $d$ exceeds  a astronomically large  we will only need around $4d$ bits of a fixed-length code to communicate the subgradient at each round. That is, our fixed-length requires only 4 bits dimension. Formally, we have the following corollary. 
\begin{cor}\label{c:PSGD_RATQ_0_1}Suppose we have 
$d \geq 3$ and $\ln^*(\frac{d}{3}) \leq 2^{8\cdot 10^3}$. Then,  the overall precision $r$ used by the quantizer $Q=Q_{{\tt at},R}$
  with parameters set as in~\eqref{e:RATQ_levels},~\eqref{e:RATQ_bits} satisfies
\[
r\leq d\left(4+\frac{\ceil{\log (1+\ln^*({d}/{3})) }}{d}\right).
\] 
Furthermore, the optimization protocol $\pi$
  given in Algorithm~\ref{a:SGD_Q} satisfies 
  \[\sup_{(f,O)\in \oO_0}\ep(f, \pi^{QO}) \leq \frac{\sqrt{2}DB}{\sqrt{T}}.\]
\end{cor}
\new{\begin{proof}
    To see the first statement, we have from Corollary \ref{c:PSGD_RATQ_0} that
$r \leq d(1+\Delta_1)+\Delta_2$. Note that $\Delta_1 $ is a monotonic function of $\ln^*(d/3)$. Therefore, by  upper bounding $\ln^*(d/3)$ by $2^{8\cdot 10^3}$ in the formula for $\Delta_1$ completes the proof. The second statement follows from Corollary \ref{c:PSGD_RATQ_0}. 
\end{proof}}

Thus, our proposed RATQ offers roughly the same compression as
  QSGD from \cite{alistarh2017qsgd} and, in addition, has the
  advantage of using only a fixed-length code,  whereas QSGD uses a
  variable-length code whose 
worst-case length becomes significantly higher than the worst-case length
of RATQ as $d$ increases.

\subsection{RATQ in the low-precision regime}\label{s:LowPrec}
We present a general method for reducing precision to much below
$r$. This scheme is applicable when the output of the quantizer's encoder is a $d$ length vector, where each coordinate is a separate fixed-length code. We simply reduce the length of the output message vector from the
quantizer's encoder by sub-sampling a subset of coordinates using shared
randomness.  The decoder obtains the values of these coordinates using
the decoder for the original quantizer and sets the rest of the
coordinate-values to zero.  This subsampling layer, which we call the
{\em Random Coordinate Sampler} (RCS), can be added to 
RATQ after applying random rotation. In particular, 
 $RATQ$ we need the parameter $s$ of these quantizers to be set to 1. This requirement of setting $s=1$ ensures that the subsampled coordinates of the rotated vector can be decoded separately. This is a randomized scheme and
requires the encoder and the decoder to share a random set
$S\subset[d]$ distributed uniformly over all subsets of $[d]$ of
cardinality $\mu d$.
 
The encoder $\Qenc_{S}$ of RCS simply outputs the vector
\[\Qenc_{S}(Y)\eqdef\{Y(i), i\in S\},\] and the decoder
$\Qdec_{S}(\tilde{Y})$, when applied to a vector $\tilde{Y}\in \R^{\mu
  d}$, outputs\[\Qdec_S(\tilde{Y})\eqdef\mu^{-1}\sum_{i\in
  S}\tilde{Y}(i)e_i,\] where $e_i$ denotes the $i$th element of
standard basis for $\R^d$.

We can compose RCS with  RATQ with parameter $s=1$  by setting the encoder to $\Qenc_{S}\circ \Qenc$, and setting the
decoder to $\Qdec\circ \Qdec_{S}$. Here we follow the convention that
all $0$-coordinates outputted by $\Qdec_{S}$ are decoded as $0$ by
$\Qdec$.    
 Note that 
since we need to retain RATQ encoder output for only $\mu d$ coordinates, the overall precision of the quantizer is reduced by a factor of $\mu$. \new{We analyze the performance of this combined quantizer in the following theorem.
\begin{thm}\label{t:RCS_RAQ_alpha_beta}
Let $Q_{{\tt at}, R}$ be RATQ  with $s=1$ and $\tilde{Q}$ be the combination of RCS
and $Q_{{\tt at}, R}$ as described above. Then,

\[
\E{\tilde{Q}(Y) | Y }   = \E{
  Q_{{\tt at}, R}(RY)| Y}   \quad \text{and} \quad  \E{\norm{\tilde{Q}(Y) }_2^2| Y }   = \frac{1}{\mu}\E{\norm{Q_{{\tt at}, R}(RY)}_2^2 |Y},
\] 
which further leads to
\[ \alpha_0(\tilde{Q})\leq
\frac{\alpha_0(Q_{{\tt at}, R})}{\sqrt{\mu}} \quad \text{and}
\quad \beta_0(\tilde{Q})=\beta_0(Q_{{\tt at}, R}).\] 
\end{thm}
}
\begin{proof}
By the description of $Q_{{\tt at},R}$, we have  
\[
\tilde{Q}(Y) = \frac{1}{\mu}
      R^{-1} \sum_{i \in S} Q_{{\tt at},I}(RY) (i)e_i,
\]
\newest{where $Q_{{\tt at},I}$ is the output vector formed by combining the $d$
  quantized values outputted by ATUQ ($Q_{\tt at}$) when input is the 
rotated vector. Namely,   
\[
Q_{{\tt at}, I}(RY)=[Q_{{\tt at}}(RY(1)), \cdots,  Q_{{\tt
      at}}(RY(d))]^T.
\]}

For the mean of $\tilde{Q}(Y)$, 
it holds that
\begin{align}\label{e:mean_RCS}
 \nonumber
\E{\tilde{Q}(Y) | Y } & = \E{
      R^{-1} \sum_{i \in d} Q_{{\tt at}, I}(RY)(i) e_i \frac{1}{\mu}\mathbbm{1}_{i \in S} |  Y}\\ \nonumber
  & = \sum_{i \in d} \E{
      R^{-1}  Q_{{\tt at}, I}(RY) (i)e_i|  Y}\cdot   \frac{1}{\mu} \E{\mathbbm{1}_{i \in S}|  Y }\\ \nonumber
      & = \sum_{i \in d} \E{
      R^{-1}  Q_{{\tt at}, I}(RY) (i)e_i| Y} \\ \nonumber
       & =  \E{
  R^{-1} \sum_{i \in d} Q_{{\tt at}, I}(RY) (i)e_i| Y}  \\ \nonumber
       & =  \E{
  R^{-1} Q_{{\tt at}, I}(RY)| Y}  \\
     & =  \E{
Q_{{\tt at}, R}(RY)| Y},
\end{align}
where the second identity follows from the fact that randomness used
to generate a set $S$ is independent of the randomness
used in the quantizer and the randomness of $Y$;  the third identity holds since
$P(i \in S)= \mu$. 

Next, moving to the computation of the second moment of the output of $\tilde{Q}$,
we have
\begin{align} \nonumber
\E{\norm{\tilde{Q}(Y)}_2^2|Y} &=\E{\norm{ \frac{1}{\mu}
      R^{-1} \sum_{i \in S} Q_{{\tt at}, I}(RY)(i) e_i}_2^2|Y} \\ \nonumber
   &=   
      \frac{1}{\mu^2}\E{\norm{
     \sum_{i \in S} Q_{{\tt at}, I}(RY)(i) e_i}_2^2|Y}\\ \nonumber
&=\frac{1}{\mu^2}\sum_{i \in [d]}\E{Q_{{\tt at}, I}(RY)(i)^2 |Y}\E{\mathbbm{1}_{i \in S}|Y}\\ \nonumber
&=\frac{1}{\mu}\E{\norm{Q_{\tt at}(RY)}_2^2 |Y} \\ 
&=\frac{1}{\mu}\E{\norm{Q_{{\tt at}, R}(RY)}_2^2 |Y},
\end{align}
where the second identity follows from the fact that $R$ is a unitary matrix and
the remaining steps follow simply by the description of the quantizers used.
It follows that
\eq{
\alpha(\tilde{Q})= \frac{1}{\sqrt{\mu}}\alpha(Q_{{\tt at}, R}),\quad  \beta(\tilde{Q})=\beta(Q_{{\tt at}, R}).
}
\end{proof}


{We now set  the parameter  $k$ to be a constant and sample roughly $r$ coordinates. Specifically,
we set
\begin{align}
  s&=1,\quad \log( k+1)=3,
\nonumber
\\ 
\mu d&=\min\{d, \floor{r/({3+\ceil{\log (1+\ln^*({d}/{3}) )
  }})}\}.
\label{e:RATQ_RCS_params}
\end{align}

For these choices, we have the following corollary.
 \begin{cor}\label{c:PSGD_RCS_RATQ}
For $r \geq 3+\ceil{\log (1+\ln^*({d}/{3}) ) }$, let $Q$
be the composition of RCS and RATQ with parameters set as in~\eqref{e:RATQ_levels},~\eqref{e:RATQ_RCS_params}.
 Then, the optimization protocol $\pi$ in
Algorithm~\ref{a:SGD_Q} satisfies
 \[\sup_{(f, O)\in \oO_0} \ep(f, \pi^{QO}) \leq 
    \frac{\sqrt{2}DB}{\sqrt{\mu T}}\]
\end{cor}
\begin{proof}

When $Q$ is a composition of 
 RCS and RATQ, from Theorem \ref{t:RCS_RAQ_alpha_beta}
 $
 \alpha(Q) \leq \frac{1}{\sqrt{\mu}}\alpha(Q_{{\tt at}, R}), \quad \beta(Q) \leq  \beta(Q_{{\tt at}, R}),
$
which by Theorem \ref{t:basic_convergence} yields
\eq{
\sup_{(f, O) \in \oO_0}\ep(f, \pi^{QO}) &
\leq
D\left(\frac{\alpha_0(Q_{{\tt at}, R})}{\sqrt{\mu T}}+ \beta_0(Q_{{\tt at}, R})\right)\\
&\leq \frac{DB}{\sqrt{\mu T}}\cdot\sqrt{ \frac{9}{(k-1)^2}+1 }\hspace{0.1cm}\\
&\leq \frac{\sqrt{2}DB }{\sqrt{T}}\cdot\frac{\sqrt{d}}{\sqrt{\min{\{d, \floor{r/(3+\log \ln^*(d/3))}\}} }},
}
where the second  inequality follows from Theorem \ref{t:RATQ_alpha0_beta0} with $s=1$, and the final inequality is obtained upon substituting the parameters as in the statement of the result. 

\end{proof}
\begin{rem}
Note that the convergence rate slows down by a $\mu$ specified in~\eqref{e:RATQ_RCS_params}, which matches the lower bound in Theorem \ref{t:e_r_LB} upto a multiplicative factor of $O(\log \ln ^* (d/3))$
\end{rem}
}

%% file: meansquare.tex
\section{Main results for mean square bounded oracles}\label{s:ms} 

{
Moving to oracles satisfying the mean square  bounded assumption, 
we now need to quantize 
random vectors $Y$ such that $\E{\norm{Y}_2^2} \leq B^2$. 
We take recourse to the standard {\em gain-shape} quantization
paradigm in vector quantization ($cf.$\cite{gersho2012vector}). 
\begin{defn}[Gain-shape quantizer]\label{Gain-shape quantizer}
A Quantizer Q is defined to be a gain-shape quantizer if it has the following form
\[Q(Y)=Q_g(\norm{Y}_2)\cdot Q_s(Y/\norm{Y}_2),\]
where $Q_g$ is any $\R \rightarrow \R$ quantizer and $Q_s$ is any $\R^d \rightarrow \R^d$ quantizer.
\end{defn}
Specifically, we separately quantize the {\em gain}
$\norm{Y}_2$ and the {\em shape}{\footnote{\newest{For the event $\norm{Y}_2=0$, we  follow the convention that $Y/\norm{Y}_2=e_1$.}} $Y/\norm{Y}_2$ of $Y$, and form the
estimate of $Y$ by simply multiplying the estimates for the gain and
the shape.
 Note that we already have a good shape quantizer: RATQ. 
We only need to modify the parameters in~\eqref{e:RATQ_levels} to make
it work for the unit sphere; we set
\begin{align}
m=\frac{3}{d}, \quad m_0=\frac{2}{d} \cdot \ln s,  \quad \log h=\ceil{\log(1+\ln^\ast(d/3))}.
\label{e:RATQ_unit_levels}
\end{align}

}
\new{
\newest{
  We now proceed to derive the worst-case $\alpha$ and $\beta$ for  a general gain-shape. In order to make clear the dependence on $B$ and $d$, we refine our notations for  $\{\alpha(Q), \beta(Q)\}$ and $\{\alpha_0(Q), \beta_0(Q)\}$, defined in \eqref{e:alpha,beta} and \eqref{e:alpha0,beta0}, respectively, to   $\{\alpha(Q; B, d), \beta(Q; B, d)\}$  and  $\{\alpha_0(Q; B, d), \beta_0(Q; B, d)\}$.
  }

\begin{thm}\label{t:gain_RATQ}
Let $Q(Y)= Q_{1}(\norm{Y}_2)\cdot Q_{2}(Y/\norm{Y}_2) ,$ where $Q_{1}$ is any gain quantizer and   $Q_{2}$ is any shape quantizer.
Also, suppose $Q_{1}(\norm{Y}_2)$ and  $Q_{2}(Y/\norm{Y}_2)$ are conditionally independent given $Y$. Then,
\begin{align*}  
\alpha(Q; B, d)&\leq   \alpha(Q_{1}; B, 1) \cdot \alpha_0(Q_{ 2};  1, d).
\end{align*}
Furthermore, suppose that $Q_2$ satisfies
\[
\E{Q_{2}(y_s) } = y_s, \quad \forall y_s \quad  s.t. \quad \norm{y_s}_2^2 \leq  1 .
\]
Then, we have\footnote{\newest{The quantity on the right-side of this bound exceeds the bias $\beta(Q_1; B,1)$. Nonetheless, in all our bounds for bias, this is the quantity we have been handling.}}
\eq{\beta(Q; B, d)&\leq \sup_{Y \in \R^d: \E{\norm{Y}_2^2} \leq B^2 }\E{\bigg|\E{Q_{1}(\norm{Y}_2)- \norm{Y}_{2} \mid Y} \bigg|}.}

\end{thm}
}
\new{\begin{proof}
    Denote by $Y_s$ the shape of the vector $Y$ given by
    \[
    Y_s := \frac{Y}{\norm{Y}_2}.
    \]
\paragraph{The worst-case second moment:}
Towards evaluating $\alpha(Q; B, d)$, we have 
\eq{
\E{\norm{Q(Y)}_2^2} 
& =\E{Q_{1}(\norm{Y}_2)^2 \norm{Q_{2}(Y_s)}_2^2} \\
&=\E{\E{Q_{1}(\norm{Y}_2)^2 \norm{Q_{2}(Y_s)}_2^2|Y}}\\
&=\E{\E{Q_{1}(\norm{Y}_2)^2 |Y} \E{ \norm{Q_{2}(Y_s)}_2^2|Y}}\\
&=\E{\E{Q_{1}(\norm{Y}_2)^2 |Y} \E{ \norm{Q_{2}(Y_s)}_2^2|Y_s}},
}
where the third identity follows by conditional independence of $Q_{1}(\norm{Y}_2)^2$ and $ \norm{Q_{2}(Y_s)}_2^2$
given $Y$ and the fourth follows from the  law of iterated expectations. 

Consider the random variable $\E{ \norm{Q_{2}(Y_s)}_2^2|Y_s}$. We claim that this is less than $\alpha_{0}(Q_{2}; 1, d)$ almost surely. Towards this end, note that
\[
\E{ \norm{Q_{2}(Y_s)}_2^2|Y_s=y} =\E{\norm{Q_{2}(y)}_2^2},
\]
since the randomness used in implementation of $Q_2$ is independent of the input random variable $Y$.
Moreover, for any $y$ with $\norm{y}_2^2 \leq 1$, we have from the definition of $\alpha_0(Q_{2}; 1, d)$ that
$\E{\norm{Q_{{2}}(y)}_2^2 }  \leq  \alpha_0(Q_{2}; 1, d)^2$.
Therefore, for any $Y$ with $\E{\norm{Y}_2^2} \leq B^2$, we have
\begin{align}\label{e:cond_secmom_gain-RATQ}
\nonumber
\E{\norm{Q(Y)}_2^2} &=\E{\E{Q_{1}(\norm{Y}_2)^2 |Y} \E{ \norm{Q_{2}(Y_s)}_2^2|Y_s}}\\
&\leq \E{\E{Q_{1}(\norm{Y}_2)^2 |Y} } \cdot \alpha_0(Q_{2}; 1, d)^2
\nonumber
\\
&= \E{Q_{1}(\norm{Y}_2)^2 } \cdot \alpha_0(Q_{2}; 1, d)^2
\nonumber
\\
&\leq \alpha(Q_{1}; B, 1)^2 \cdot \alpha_0(Q_{2}; 1, d)^2
 .
 \end{align}
Taking the supremum  of the left-side over all random vectors $Y$ with $\E{\norm{Y}_2^2}\leq B^2$ gives us the desired bound for $\alpha(Q)$.

\paragraph{The worst-case bias:}  Towards evaluating $\beta(Q)$, we note from  our hypothesis that
 $\E{ Q_{2}(Y_s)| Y} = \E{ Q_{2}(Y_s)| Y_s} = Y_s$, which further yields 
\begin{align} \label{e:RCSbias-gain-ratq}
\E{Q(Y)-Y} &=\E{\E{Q_{1}(\norm{Y}_2) Q_{ 2}(Y_s)-Y|Y}} \nonumber
\\ \nonumber
&=\E{\E{Q_{1}(\norm{Y}_2)|Y} \E{Q_{2}(Y_s)|Y}-Y}\\
&=\E{\E{Q_{1}(\norm{Y}_2)|Y}  Y_s -\norm{Y}_2Y_s} \nonumber
\\
&=\E{\E{Q_{1}(\norm{Y}_2)- \norm{Y}_{2} |Y} Y_s},
\end{align}
where the second identity uses conditional independence of 
$Q_{1}(\norm{Y}_2)$ and $Q_{ 2}(Y_s)$.
By using the conditional Jensen's inequality, we get
 \eq{ 
\norm{ \E{Q(Y)-Y}}_2 &=
\norm{\E{\E{Q_{1}(\norm{Y}_2)- \norm{Y}_{2} |Y} Y_s}}_2\\
&\leq \E{\norm{\E{Q_{1}(\norm{Y}_2)- \norm{Y}_{2} |Y} Y_s}_2}\\
&= \E{\bigg|\E{Q_{1}(\norm{Y}_2)- \norm{Y}_{2} |Y} \bigg|}.}
 \end{proof}

}
}

We remark that
quantizers proposed in most of the prior work can be cast in this
gain-shape framework. Most works simply state that gain is a
single parameter which can be quantized using a fixed number of bits;
for instance, a single double precision number is prescribed for
storing the gain. However, the quantizer is not specified. We carefully
analyze this problem and establish lower bounds when a uniform quantizer
with a fixed dynamic range is used for quantizing the gain. Further, we
present our own quantizer which significantly outperforms  uniform gain
quantization.

\subsection{Limitation of uniform gain quantization}\label{s:ug}
We establish lower bounds for a general class of gain-shape quantizers 
$Q(y)=Q_g(\norm{y}_2)Q_s(y/\norm{y}_2)$ of precision $r$ that satisfy the following {\em structural assumptions}: 
\begin{enumerate}\label{en:Assumptions}
\item {\bf (Independent gain-shape quantization)} For any given $y\in \R^d$, the output of the gain and the shape quantizers are independent. 

\item {\bf (Bounded dynamic-range)} There exists $M>0$ such that $y\in \R^d$ such that whenever $\norm{y}_2> M$, $Q(y)$ has a fixed distribution $P_\emptyset$. 
\item {\bf (Uniformity)} There exists $m \in [M/{2^r}, M]$ such that for every  $t$ in $[0,  m ]$,
\begin{enumerate}
\item ${\tt supp}(Q_g(t))\subseteq \{0,m\}$;
\item If $P(Q_g(t)=m)>0$, then 
\[
\frac{P(Q_g(t_2)=m)}{ P(Q_g(t_1)=m)}\leq~ \frac{t_2}{t_1}, \quad \forall\, 0\leq t_1 \leq t_2 \leq m.
\]
\end{enumerate}
\end{enumerate}
The first two assumptions are perhaps clear and hold for a large class of quantizers. The third one is the true limitation and is satisfied by different forms of uniform gain quantizers. For instance, for the one-dimensional version of CUQ  with dynamic range $[0,M]$, which is an unbiased, uniform gain quantizer with $k_g$ levels, it holds with $m=M/(k_g-1)$ (corresponding to the innermost level $[0,M/(k_g-1)]$). It can also be shown to include a deterministic uniform quantizer that rounds-off at the mid-point. The third condition, in essence, captures the unbiasedness requirement that the probability of declaring higher level is proportional to the value. Note that $(t_2/t_1)$ on the right-side can be replaced with any constant multiple of $(t_2/t_1)$.

Below we present lower bounds for performance of any optimization
protocol using a gain-shape quantizer that satisfies the assumptions
above. We present separate results for high-precision and
low-precision regimes, but both are obtained using a 
general construction that 
exploits the admissibility of heavy-tail  
 distributions for mean square  bounded oracles. This  construction
  is new and may be of independent interest.   
 
\begin{thm}\label{t:lb_1}
Consider a gain-shape quantizer $Q$ satisfying the assumptions above. 
 Suppose that for $\X=\{x:\norm{x}_2\leq D/2\}$ we can find an optimization protocol
$\pi$ which, using at most $T$ iterations, achieves
$
\sup_{f, O \in \oO}\mathcal{E}(f, \pi^{QO}) \leq
\frac{3DB}{\sqrt{T}}.
$
Then, we can find a universal constant $c$ such that the overall precision $r$ of the quantizer must satisfy 
\[
r \geq c(d+\log T).
\] 
 \end{thm}

 \begin{thm}\label{t:lb_2}
Consider a gain-shape quantizer $Q$ satisfying the assumptions above
. Suppose that the number of bits \newest{$r_g$} used by the gain quantizer are fixed
independently of $T$. Then, for
$\X=\{x: \norm{x}_2\leq D/2\}$, there exists $(f,O)\in \oO$ such that
for any optimization protocol $\pi$ using at most $T$ iterations, we
must have
\[
\mathcal{E}(f, \pi^{QO}) \geq \frac{\newest{c(r_g)}DB}{T^{1/3}},
\]
where \newest{$c({r_g})$} is a constant depending only on the number of bits used by
the gain quantizer (but not on $T$).
 \end{thm}
 \noindent The proofs of Theorems~\ref{t:lb_1} and~\ref{t:lb_2} are technical and long; we defer
them to Section \ref{s:lbproof}.

 \subsection{A-RATQ in the high precision regime}\label{s:ARATQh}
Instead of quantizing the gain uniformly, we propose to use an
adaptive quantizer termed {\em Adaptive Geometric Uniform Quantizer} (AGUQ) for
gain. AGUQ operates similar to the one-dimensional
ATUQ, except the possible dynamic-ranges $M_{g,0}, \ldots, M_{g,h}$ grow
geometrically (and not using tetra-iterations of ATUQ) 
as follows:
\begin{align}
M_{g,j}^2= B^2 \cdot a_g^{j}, \quad 0\leq j \leq h_g-1.
\label{e:AGUQ_levels}
\end{align}
Specifically, for a given gain $G\geq 0$, AGUQ first identifies the
smallest $j$ such that $G\leq M_{g,j}$ and then represents $G$ using the one-dimensional version of CUQ with a dynamic range $[0, M_{g,j}]$ and
$k_g$ uniform levels
\eq{ &B_{M_{g,j},k}(\ell) := \ell \cdot
    \frac{M_{g, j}}{k_g-1}, \quad \forall\,  \ell \in \{0,\ldots,
    k-1\}.} 
 As in ATUQ, if $G>M_{h_g -1}$,
the overflow $\emptyset$ symbol is used and the decoder simply outputs
$0$. The overall procedure is the similar to
Algorithms~\eqref{a:E_ATQ} and~\eqref{a:D_ATQ} for $s=1, h=h_g$,  and
$M_j=M_{g,j}$, $0\leq j\leq h_g-1$; the only changes is that now we
restrict to nonnegative interval $[0,M_{g,j}]$ for the one-dimensional version of CUQ with uniform levels $k_g$.
\new{
\newest{The following result characterizes the performance of one-dimensional quantizer AGUQ; it is
the only component missing in the analysis of A-RATQ.}
\begin{lem}\label{l:AGUQ_alpha_beta}
Let $Q_{{\tt a}}$ be
the quantizer AGUQ described above, with $h_g \geq 2$. Then,
\begin{align*}
  &\alpha(Q_{{\tt a}}; B, 1)\leq B\sqrt{ \frac{1 }{4(k_g-1)^2} + 
    {\frac{a_g(h_g-1)}{4(k_g-1)^2}} + 1 },\\
  &\beta(Q_{{\tt a}}; B, 1) \leq \sup_{Y \geq 0 \text{ a.s.  } : \E{Y^2} \leq B^2 }\E{\bigg|\E{Q_{1}(Y)- Y |Y} \bigg|}
\leq \frac{B^2}{M_{g, h_g-1}}.
\end{align*}
\end{lem}
}
\noindent Proof of this result, too, is deferred to Section~\ref{s:AGUQ}. Note that we have derived a bound for a quantity that is slightly larger than the bias of $Q_{\tt a}$, since we want to use this result along with Theorem~\ref{t:gain_RATQ}.  

Thus, our overall quantizer termed the {\em adaptive-RATQ} (A-RATQ) is given by 
\[Q(Y) := Q_{a}(\norm{Y}_2) \cdot Q_{{\tt at}, R}(Y/\norm{Y}_2),\]
where $Q_{a}$ denotes the one dimensional AGUQ and $Q_{{\tt at}, R}$
denotes the $d\text{-dimensional}$ RATQ. Note that we use independent
randomness for  $Q_{a}(\norm{Y}_2)$ and $Q_{{\tt at},
  R}(Y/\norm{Y}_2)$, rendering them  conditionally independent given $Y$. 

The parameters $s,k$ for RATQ and $a_g, k_g$ for AGUQ are yet
to be set. We first present a result which holds for all choices of
these parameters.

\mnote{$Q_{at, I}$ instead of $Q_{at}$}
\begin{thm}[Performance of A-RATQ]\label{t:NRATQ_alpha_beta}
For $Q$ set to A-RATQ with parameters set as
in~\eqref{e:RATQ_unit_levels},~\eqref{e:AGUQ_levels}, we have 
\begin{align*}  
&\alpha(Q; B, d)\leq  B\sqrt{ \frac{1}{4(k_g-1)^2} + 
    {\frac{a_g(h_g-1)}{4(k_g-1)^2}} + 1 } \cdot \sqrt{
\frac{9+3\ln s}{(k-1)^2}+1} ,\\
  &\beta(Q; B, d)\leq \frac{B^2}{M_{g, h_g-1}}.
\end{align*}
\end{thm}
\begin{proof}~\paragraph{The worst-case second moment  of A-RATQ:}
By Theorem~\ref{t:gain_RATQ} we have
 \eq{
 \alpha(Q; B, d) 
&\leq  \alpha(Q_{\tt a}; B, 1) \cdot   \alpha_0(Q_{{\tt at}, R}; 1, d)\\
& \leq \alpha(Q_{\tt a}; B, 1) \cdot \sqrt{
\frac{9+3\ln s}{(k-1)^2}+1}\\
& \leq B\sqrt{ \frac{1}{4(k_g-1)^2} + 
    {\frac{a_g(h_g-1)}{4(k_g-1)^2}} + 1 }  \cdot \sqrt{
\frac{9+3\ln s}{(k-1)^2}+1} ,
}
where the second inequality used Theorem~\ref{t:RATQ_alpha0_beta0} with $B=1$,
and the third follows by Lemma~\ref{l:AGUQ_alpha_beta}. 
\paragraph{The worst-case bias of A-RATQ:}
With parameters of RATQ set as in~\eqref{e:RATQ_unit_levels}, we have that 
\eq{ 
\E{Q_{{\tt at}, R}(y)}=y, \quad \forall y \quad s.t \quad \norm{y}_2^2 \leq 1. }
 Therefore, by Theorem~\ref{t:gain_RATQ} it follows that
 \eq{ 
\beta(Q; B, d)&
\leq
\sup_{Y: \E{\norm{Y}_2^2} \leq B^2 }\E{\bigg|\E{Q_{\tt a}(\norm{Y}_2)- \norm{Y}_{2} |Y} \bigg|}
\leq \frac{B^2}{M_{g,h_g-1}},}
 where the second  inequality follows from Lemma \ref{l:AGUQ_alpha_beta}.
\end{proof}

Note that RATQ yields an unbiased estimator; the bias in A-RATQ arises
from AGUQ since the gain is not bounded. Further, AGUQ uses a precision
of $\ceil{\log h_g} + \ceil{\log (k_g+1)}$ bits, and therefore, the
overall precision of A-RATQ is $\ceil{\log h_g} + \ceil{\log (k_g+1)}+
\ceil{d/s}\ceil{\log h}+ d \ceil{\log (k+1)}$ bits.

In the high-precision regime, we set 
\begin{align}
a_g&=2,\quad \log h_g =
\ceil{\log(1+\frac{1}{2}\log T)}, 
\nonumber
\\
\log(k_g+1)&=\ceil{\log \left( 2+\frac{1}{2}\sqrt{\log T +1} \right)}.
\label{e:ARATQ_gain_bits}
\end{align}

\begin{cor}\label{c:PSGD_NRATQ} Denote by $Q$ the quantizer A-RATQ
  with parameters set as
  in~\eqref{e:RATQ_unit_levels},~\eqref{e:RATQ_bits},
  and~\eqref{e:ARATQ_gain_bits}. Then,  
the overall precision $r$ used by $Q$
 is less than 
\[
d(1+\Delta_1)+\Delta_2+  \ceil{\log \bigg( 2+\sqrt{\log T +1} \bigg)},
\] 
where $\Delta_1=\ceil{\log \left(2+  \sqrt{  9 + 3 \ln \Delta_2 }\right)}$ and 
$\Delta_2=\ceil{\log (1+\ln^*({d}/{3})) }$, the same as
Corollary~\ref{c:PSGD_RATQ_0}. 
Furthermore, the optimization protocol $\pi$
  given in algorithm \ref{a:SGD_Q} satisfies
 $\sup_{(f,O)\in \oO}  \ep(f, \pi^{QO}) \leq {3DB}/{\sqrt{T}}$.
\end{cor}
\new{\begin{proof}
The proof is similar to the proof of Corollary \ref{c:PSGD_RATQ_0}. The first statement follows by simply upper bounding the precision of the fixed-length code for A-RATQ with parameters as in the statement. The second statement follows by bounding  
$\sup_{(f,O)\in \oO}  \ep(f, \pi^{QO})$ using Theorem \ref{t:basic_convergence}, using the upper bounds for $\alpha$ and $\beta$ given in Theorem~\ref{t:NRATQ_alpha_beta}, and finally substituting the parameters.
\end{proof}
}
\new{
\begin{rem}
The  precision used in Corollary \ref{c:PSGD_NRATQ} 
  matches the lower bound in Corollary \ref{c:OmegaD} upto an additive
  factor  of $\log \log T$ (ignoring the mild factor of  $\log \log
  \ln^* (d/3)$), which is much lower than the $\log T$ lower bound we established for uniform
  gain quantizers.
\end{rem}
}
 \subsection{A-RATQ in the low precision regime}\label{s:ARATQl}
  In order to operate with a fixed precision $r$, we combine
 A-RATQ with RCS. 
  We simply combine RCS with RATQ as 
 in Section \ref{s:LowPrec} to limit the precision and use AGUQ as the gain quantizer. Note that we use independent
randomness in our gain quantizer  $Q_{a}(\norm{Y}_2)$ and our shape quantizer $\tilde{Q}(Y/\norm{Y}_2)$, rendering them  conditionally independent given $Y$. We have the following theorem characterizing $\alpha$ and $\beta$ for this  quantizer. 

 \begin{thm}\label{t:RCS_Gain-RATQ}
Let $Q(Y)=Q_a(\norm{Y})\cdot \tilde{Q}(Y/\norm{Y}_2)$, where $\tilde{Q}$ is the composition of RCS and RATQ described in Theorem \ref{t:RCS_RAQ_alpha_beta} with parameters $m$, $m_0$, and $h$ of RATQ as in \eqref{e:RATQ_unit_levels}  and $Q_a$ is AGUQ.  Then, 
\eq{ \alpha(Q; B, d)&\leq  B\sqrt{ \frac{1}{4(k_g-1)^2} + 
    {\frac{a_g(h_g-1)}{4(k_g-1)^2}} + 1 }  \cdot \frac{1}{\sqrt{\mu}}  \sqrt{
\frac{9+3\ln s}{(k-1)^2}+1 },\\
  \beta(Q; B, d)&\leq \frac{B^2}{M_{g, h_g-1}}.}

\end{thm}
\begin{proof}~\paragraph{The worst-case second moment:} 
Starting by applying Theorem~\ref{t:gain_RATQ}, we have
 \eq{
 \alpha(Q; B, d) 
 &\leq  \alpha(Q_{\tt a}; B, 1) \cdot   \alpha_0(\tilde{Q}; 1, d)\\
&\leq  \alpha(Q_{\tt a}; B, 1) \cdot   \frac{1}{\sqrt{\mu}}\alpha_0(Q_{{\tt at}, R}; 1, d)\\
& \leq B\sqrt{ \frac{1}{4(k_g-1)^2} + 
    {\frac{a_g(h_g-1)}{4(k_g-1)^2}} + 1 }  \cdot \frac{1}{\sqrt{\mu}}\sqrt{
\frac{9+3\ln s}{(k-1)^2}+1} ,
}
where the second inequality follows by Theorem~\ref{t:RCS_RAQ_alpha_beta} and the third follows by Theorem~\ref{t:RATQ_alpha0_beta0} and Lemma~\ref{l:AGUQ_alpha_beta}.
\paragraph{The worst-case bias:} With parameters of RATQ set as in~\eqref{e:RATQ_unit_levels}, we have that 
\eq{ 
\E{\tilde{Q}(y)}=y, \quad \forall y \quad s.t \quad \norm{y}_2^2 \leq 1. }
 Therefore, by Theorem \ref{t:gain_RATQ} we get
 \eq{ 
\beta(Q; B, d)&
\leq
\sup_{Y: \E{\norm{Y}_2^2} \leq B^2 }\E{\bigg|\E{Q_{\tt a}(\norm{Y}_2)- \norm{Y}_{2} |Y} \bigg|}\\
&\leq \frac{B^2}{M_{g,h_g-1}},}
 where the second  inequality follows from Lemma \ref{l:AGUQ_alpha_beta}.

\end{proof}
 
 We divide the total precision $r$ into $r_g$ and $r_s$ bits: $r_g$ to quantize the gain, $r_s$ to quantize the subsampled shape vector.
We set  
\begin{align}
 & s, k, \text{~and~} \mu d \text{~as in \eqref{e:RATQ_RCS_params}, with $r_s$ replacing $r$,}
\nonumber 
 \\ 
&{\log h_g}=\log (k_g+1)={\frac{r_g}{2}}, 
~a_g=\left(\mu T\right)^{\frac{1}{h_g+1}}
\label{e:ARATQ_RCS_params}
\end{align}
That is, our shape quantizer simply quantizes $\mu d$ randomly chosen
coordinates of the rotated vector using ATUQ with $r_s$ bits, and the
remaining bits are used by the gain quantizer AGUQ. 
The result below shows the performance of this quantizer.

 \begin{cor}\label{c:PSGD_RCS_ARATQ_fixed}
For any $r$ with gain quantizer being assigned  $r_g \geq 4$ bits and shape quantizer being assigned $r_s \geq 3+\ceil{\log (1+\ln^*({d}/{3}) ) }$, let $Q$
be the combination of RCS and A-RATQ with parameters set as
in~\eqref{e:RATQ_unit_levels},~\eqref{e:AGUQ_levels},~\eqref{e:ARATQ_RCS_params}. Then
for $\mu T\geq 1$, the optimization protocol $\pi$ in
Algorithm~\ref{a:SGD_Q} 
can obtain  
\[\sup_{(f, O)\in \oO} \ep(f, \pi^{QO})   \leq
O\left(DB \left(\frac{d}{T \min\{d, \frac{r_s}{\log \ln^*(d/3)}\}}\right)^{\frac 12\cdot\frac {2^{r_{g}/2}-1}{2^{r_{g}/2}+1}} \right).
\]
 \end{cor}
 \begin{proof} By using Theorem~\ref{t:basic_convergence} to upper bound $\sup_{(f, O) \in \oO}\ep(f, \pi^{QO})$ and then Theorem \ref{t:RCS_Gain-RATQ} to upper-bound $\alpha$ and $\beta$, we get
\eq{
\sup_{(f, O) \in \oO}\ep(f, \pi^{QO}) &
\leq  D\left(\frac{1}{\sqrt{\mu T}}\sqrt{ \frac{B^2}{4(k_g-1)^2}+\frac{a_g(h_g-1)B^2}{4(k_g-1)^2}
+B^2}\sqrt{
\frac{9+3\ln s}{(k-1)^2}+1}+ \frac{B^2}{M_{g, h-1}}\right).
}
 By substituting the parameters as in the statement and using the fact that $\mu T \geq 1$ completes the proof.
\end{proof}
 
 \begin{rem}
 \new{Our fixed precision quantizer in
  Corollary~\ref{c:PSGD_RCS_ARATQ_fixed} establishes that using only a constant number of bits for gain-quantization, we get very close to the lower  bound in Theorem~\ref{t:e_r_LB}. For instance, given access to a large enough precision $r$,  if we set $r_g$ to be $16$ bits, we get
  \[\sup_{(f, O)\in \oO} \ep(f, \pi^{QO})   \leq
O\left(DB \left(\frac{d}{T \min\{d, \frac{r -16}{\log \ln^*(d/3)}\}}\right)^{\frac{1}{2}\cdot \frac{255}{257}}\right).
\]
Here, the ratio of $d/(\min\{d, \frac{r -16}{\log \ln^*(d/3)}\})$ is very close to the optimal ratio of $d/(\min\{d, r\})$, and the exponent $255/(2\cdot 257)$ is close to the optimal exponent $1/2$.
 }
  \end{rem}

\begin{rem}
  We remark that A-RATQ satisfies Assumptions (1) and (2) in Section \ref{s:ug} but not (3), and breaches the lower bound  for uniform gain quantizers established in Section~\ref{s:ug}.
  \end{rem}

%% file: proof.tex
\section{Main proofs}\label{s:proof}
\subsection{Proof of Theorem~\ref{t:RATQ_alpha0_beta0}}\label{s:ProofRATQ}
\paragraph{Step 1: Analysis of CUQ.}
We first prove a result for CUQ (with a dynamic range of $[-M, M]$)
which will bound the expected value of
\[
\sum_{i\in[d]}\big(Q_{{\tt
    u}}(Y)(i)-{Y(i)}\big)^2\mathbbm{1}_{\{{Y(i)}\leq M\}},
\]
namely the mean square error when there is no overflow. This will be
useful in the analysis of RATQ, too.     
\begin{lem}\label{l:sup}
  For an $\R^d$-valued random variable $Y$ and $Q_{{\tt u}}$ denoting the quantizer CUQ  
  with parameters $M$ (with dynamic range $[-M,M]$)
 and $k$, let $Q_{{\tt u}}(Y)$ be the quantized value of $Y$.  Then, 
 \eq{ \E{ \sum_{i\in[d]}\big(Q_{{\tt
        u}}(Y)(i)-{Y(i)}\big)^2\mathbbm{1}_{\{{|Y(i)|}\leq M\}} \mid Y
   } \leq \frac{dM^2}{(k-1)^2} \left( \frac{1}{d}\sum_{j\in[d]}
   \mathbbm{1}_{\{{|Y(j)|}\leq M\}}\right).  }
\end{lem}
\noindent The proof is relatively straightforward with the
calculations similar to \cite[Theorem 2]{suresh2017distributed}; it is
deferred to Appendix \ref{ap:CUQ}.

Also, the quantizer AGUQ in Section~\ref{s:ARATQh} uses the
one-dimensional CUQ with dynamic range $[0, M]$ as a subroutine. The
uniform levels for this variant of CUQ are given by
\[
B_{M,k}(\ell)=\ell\cdot\frac{M}{k-1}, \forall  \ell \in [k-1].
\]
We have the following lemma for this variant of CUQ.
\begin{lem}\label{l:sup0m}
  For an $\R$-valued random variable $Y$ which is almost surely nonnegative
  and the quantizer $Q_{\tt u}$ with dynamic range $[0,M]$ and parameter $k$, let $Q_{{\tt u}}(Y)$
denote the quantized value of $Y$. Then, 
\eq{
  \E{\big(Q_{{\tt
        u}}(Y)-{Y}\big)^2\mathbbm{1}_{\{{|Y|}\leq M\}} \mid Y } 
  \leq \frac{M^2}{4(k-1)^2}  \left( 
  \mathbbm{1}_{\{{|Y|}\leq M\}}\right).
} 
\end{lem}
\noindent The proof is very similar to the proof of Lemma \ref{l:sup} and is deferred to Appendix \ref{ap:CUQ}.

\paragraph{Step 2: Mean square error for adaptive quantizers.}
The quantizers RATQ and A-RATQ use ATUQ as subroutine; in addition, A-RATQ uses AGUQ for gain quantization. Thus, in order to analyze RATQ and A-RATQ,  we need to analyze
 ATUQ  and AGUQ first. 

In this step we provide a general bound on the mean square error of adaptive quantizers. We capture the performances of ATUQ and AGUQ in two separate results below. 
\begin{lem}\label{l:sup2}
For an $\R^d$-valued random variable  $Y$ and $Q$ denoting the quantizer ATUQ with dynamic-range
parameters $M_j$s, we have
\eq{
\E{ \sum_{i\in[d]}\big(Q(Y)(i)-{Y(i)}\big)^2\mathbbm{1}_{\{{|Y(i)|}\leq M_{h-1}\}}}  \leq \frac{d}{(k-1)^2}\left(m+m_0+ \sum_{j=1}^{h-1}
M_j^2P\left(\norm{Y}_{\infty} > M_{j-1}\right)\right). 
}
\end{lem}
\begin{proof}
  Consider the events $A_j$s corresponding to different levels used by
  the adaptive quantizer of the norm, defined as follows:
  \begin{align*}
    A_0 &:=   \{\norm{Y}_{\infty} \leq m \},
    \\
    A_j &:= \{M_{j-1}  <\norm{Y}_{\infty} \leq M_{j} \}, \quad \forall
    j \in [h-2],\\
     A_{h-1} &:= \{M_{h-2}  <\norm{Y}_{\infty} \}
.  \end{align*}
By construction, $\sum_{j =0}^{h-1}\indic{A_j}=1 \text{ a.s.}$. Therefore, we have
 \eq{ 
\E{\sum_{i\in[d]}\big(Q(Y)(i)-{Y(i)}\big)^2\mathbbm{1}_{\{{|Y(i)|}\leq M_{h-1}\}}}
&= \E{\norm{Q_{{\tt }}(Y)-{Y}}_2^2\indic{A_0}} + \sum_{j=1}^{h-2}\E{\norm{Q_{{\tt }}(Y)-{Y}}_2^2\indic{A_j}} \\
&\hspace{1cm} 
+\E{\sum_{i\in[d]}\big(Q_{{\tt
        u}}(Y)(i)-{Y(i)}\big)^2\mathbbm{1}_{\{{|Y(i)|}\leq M_{h-1}\}}\indic{A_{h-1}}}.
  }
Note that $ \mathbbm{1}_{A_0}$ implies
that we are using a $k$-level uniform quantization with a dynamic
range of $[-m, m]$.  Therefore, this term can be bounded by
Lemma~\ref{l:sup} as follows:
\eq{
\E{\norm{Q_{{\tt }}(Y)-{Y}}_2^2\indic{A_0}}\ \leq \frac{dm}{(k-1)^2}.
}  
Under the event $A_j$ with $j \in [h-1]$, we use   a $k$-level uniform
quantization with a dynamic range of $[-M_{j}, M_{j} ].$ Therefore, by
Lemma~\ref{l:sup}, we have 
\eq{
 \E{\norm{Q_{{\tt }}(Y)-{Y}}_2^2\indic{A_j}}
\leq & \frac{dM_j^2}{(k-1)^2}  \E{\mathbbm{1}_{A_j}}\\
\leq & \frac{dM_j^2}{(k-1)^2}  P\left(\norm{{Y}}_{\infty} > M_{j-1}\right).
}
\end{proof}
Note that the proof above does note use specific form of $M_j$'s and therefore applies as it is for the one-dimensional AGUQ gain quantizer used in A-RATQ; the only change is the fact that instead of using Lemma \ref{l:sup} for uniform quantization we use Lemma \ref{l:sup0m}. This leads to the following lemma, which will be useful later in the analysis of A-RATQ.

\begin{lem}\label{l:sup3}
For an $\R$-valued random variable $Y$ which is almost surely nonnegative
and $Q$ denoting the quantizer \newest{AGUQ} with dynamic-range
parameters $M_{g,j}$s, we have
\eq{
\E{(Q_{{\tt }}(Y)-{Y})^2\mathbbm{1}_{\{|Y| \leq
    M_{g, h-1}\}}}  \leq \frac{1}{4(k-1)^2}\left(B^2+ \sum_{j=1}^{h-1}
M_{g,j}^2P\left(|Y| > M_{g,j-1}\right)\right). 
}
\end{lem}
\noindent The proof is similar to that of Lemma~\ref{l:sup2} and is omitted.

\paragraph{Step 3: Mean square error of ATUQ for a subgaussian input vector.}
\newest{In our analysis, we need to evaluate the performance of ATUQ for {\em subgaussian} input vectors.}
\begin{defn}[$cf.$~\cite{boucheron2013concentration}]\label{d:subg}
 A centered random variable $X$ is said to be {subgaussian} with {variance factor $v$} if for all $\lambda$ in $\R$, we have
 \[
 \ln \E{e^{\lambda X}}\leq \frac{\lambda^2 v}{2}.
 \]
\end{defn}
The following well-known fact ($cf.$~\cite[Chapter 2]{boucheron2013concentration}) will be used throughout. 
 \begin{lem}\label{l:standar_subg}
For a centered subgaussian random variable $X$ with variance factor $v$ the 
\begin{align*}
&P(|X| > x) \leq 2 e^{-x^2/2v}, \quad\forall\, x>0,
\\
&\E{X^2} \leq 4 v, \quad \E{X^4} \leq 32 v^2.
\end{align*}
\end{lem} 
 

 Next, consider the quantizer $Q_{{\tt at}, I}$ which is similar to RATQ but skips the rotation step. Specifically, $Q_{{\tt at}, I}$ is obtained by replacing the random matrix $R$ in the encoder and decoder of RATQ (given in Algorithms \ref{a:E_RATQ} and \ref{a:D_RATQ}, respectively) by the identity matrix $I$. Symbolically, the quantizer $Q_{{\tt at}, I}$ can be described as follows for the $d$-dimensional input vector $Y$
 \begin{equation}\label{e:Q_at_I}
 Q_{{\tt at}, I}(Y)=[Q_{\tt at}(Y_1)^T, \cdots, Q_{\tt at}(Y_{\ceil{d/s}})^T],
 \end{equation}
where $Q_{\tt at}$ is the quantizer ATUQ and $Y_i$ is the $i^{th}$  subvector of $Y$. Recall that the $i^{th}$ subvector $Y_i$ comprises
the coordinates $\{(i-1)s+1,\cdots, \min\{is,d\}\}$,
 for all $i \in [d/s].$ Also, recall that the dimension of all the sub
 vectors except the last one is $s$, with the last one having 
 dimension $d-s\floor{d/s}$.

 Notice that like RATQ,  $Q_{at, I}$ has parameters $k$, $h$, $s$, $m$, and $m_0$ which need to be set.
 We set the parameters $m$ and $m_0$ to be $3v$ and $2v \ln s$,
 respectively, and prove a general lemma in terms of the other
 parameters of $Q_{at, I}$ for a subgaussian input vector.
\begin{lem}\label{l:subg_mse}
Consider $Y=[Y(1),\ldots, Y(d))]^T,$ where for all $i$ in $[d]$,
$Y(i)$ is a centered subgaussian random variable with variance factor
$v$. Let $Q$ denote the quantizer $Q_{{\tt at}, I}$ with parameters
$m$ and $m_0$ set to $3v$ and $2v \ln s$, respectively.
Then, for every $s,k, h \in \N$, we have 
\[
\frac{1}{d}\cdot \E{\sum_{i \in [d]}(Y(i)-Q(Y)(i))^2  \indic{\{|Y(i)|\leq M_{h-1}}\}} \leq  v \cdot \frac{9+3\ln s}{(k-1)^2}.
\] 
\end{lem} 
\begin{proof} Since 
  \eq{
    \E{\sum_{i \in [d]}(Y(i)-Q(Y)(i))^2  \indic{\{|Y(i)|\leq M_{h-1}}\}}
    &= \sum_{i=1}^{ \ceil{\frac{d}{s}}}\sum_{j =  
    (i-1)s+1}^{\min\{is,d\}}
    \E{\left(Q_{{\tt at}}(Y)(j)-Y(j)\right)^2 \indic{\{|Y(j)| \leq M_{h-1}\}}},
  }
by using Lemma $\ref{l:sup2}$ for each of the $\ceil{d/s}$
subvectors, we get
\eq{
  \nonumber
  &\lefteqn{\E{\sum_{i \in [d]}(Y(i)-Q(Y)(i))^2  \indic{\{|Y(i)|\leq M_{h-1}}\}}}
\\
  &\leq
\frac{s}{(k-1)^2}
  \sum_{i \in
  1}^{\floor{\frac{d}{s}}} \left(  m + m_0 +\sum_{j
  \in [h-1]}  M_{ j}^2 P\left(\norm{Y_{i,
    s}}_{\infty} > M_{j-1}\right) \right)
\\ \nonumber
& \hspace{0.3cm}
+\frac{(d-s\floor{\frac{d}{s}})
  }{(k-1)^2}\left(m
+  m_0+\sum_{j \in [h-1]}
  M_{ j}^2P\left(\norm{Y_{\ceil{d/s}, s}}_{\infty} >
M_{j-1}\right) \right)  .
}
For all $i \in \floor{d/s}$, it follows from the union bound that
\eq{
P\left(\norm{Y_{1,
    s}}_{\infty} > M_{j-1}\right)
\leq 2s e^{\frac{-M_{j-1}^2}{2v}}. 
}
Also, since $d-s\floor{d/s}\leq s$, we have
\eq{
P\left(\norm{Y_{\ceil{d/s},
    s}}_{\infty} > M_{j-1}\right)
\leq  2s e^{\frac{-M_{j-1}^2}{2v}}.
} 
Using these tail bounds in the previous inequality, we get
\eq{
  \nonumber \E{\sum_{i \in [d]}(Y(i)-Q(Y)(i))^2  \indic{\{|Y(i)|\leq M_{h-1}}\}} \leq \frac{d}{(k-1)^2}\left(m+m_0+2s\sum_{j \in [h-1]} M_{ j}^2 e^{\frac{-M_{j-1}^2}{2v}}\right).
} 
Setting  $m=3v$ and $m_0=2v$, the summation
on the right-side is bounded further as
\eq{
  &\lefteqn{2s\left(\frac{3v}{s }\sum_{j=1}^{h-1}(e^{*j})\cdot e^{- 1.5{e^{*(j-1)}}}\right)+2s\left(\frac{2v}{s }\sum_{j=1}^{h-1}e^{- 1.5{e^{*(j-1)}}}\right)}
  \\
&= 6v\sum_{j=1}^{h-1}e^{-0.5{e^{*(j-1)}}} +4v\ln s \sum_{j=1}^{h-1}e^{-1.5{e^{*(j-1)}}}\\
& \leq 6 v \sum_{j=1}^{\infty}e^{-0.5{e^{*(j-1)}}}+ 4v\ln s  \sum_{j=1}^{h-1}e^{-1.5{e^{*(j-1)}}}\\
& \leq 6v+v\ln s ,
}
where we use a bound of $1$ for $\sum_{j=1}^{\infty}e^{-0.5{e^{*(j-1)}}}$, whose validity
can be seen as follows\footnote{In fact, these bounds motivate the use of tetration
as our choice for $M_j$s.}
\begin{align*}
\sum_{j=1}^{\infty}e^{-0.5{e^{*(j-1)}}}
&={e^{-0.5}}+  {e^{-0.5e}}+ {e^{-0.5e^e}}
+\sum_{j=3}^{\infty}e^{-0.5{e^{*(j)}}}
\\
&\leq {e^{-0.5}}+  {e^{-0.5e}}+ {e^{-0.5e^e}}
+\sum_{j=3}^{\infty}e^{-0.5{je^e}}
\\
&\leq {e^{-0.5}}+  {e^{-0.5e}}+ {e^{-0.5e^e}}+\frac1{e^{e^e}-1}
\\
&\leq 1,
\end{align*}
and $1/4$ for $\sum_{j=1}^{h-1}e^{-1.5{e^{*(j-1)}}}$, whose validity can be seen as follows
\eq{
\sum_{j=1}^{\infty}e^{-1.5{e^{*(j-1)}}}
&={e^{-1.5}}+  {e^{-1.5e}}+ {e^{-1.5e^e}}
+\sum_{j=3}^{\infty}e^{-1.5{e^{*(j)}}}
\\
&\leq {e^{-1.5}}+  {e^{-1.5e}}+ {e^{-1.5e^e}}
+\sum_{j=3}^{\infty}e^{-1.5{je^e}}
\\
&\leq {e^{-1.5}}+  {e^{-1.5e}}+ {e^{-1.5e^e}}+\frac1{e^{3e^e}-1/e^{1.5e^e}}
\\
&\leq 0.2401.
}
Therefore, we obtain 
\[\frac{1}{d}\cdot \E{\sum_{i \in [d]}(Y(i)-Q(Y)(i))^2  \indic{|Y(i)|\leq M_{h-1}}} \leq v \cdot  \frac{9+3\ln s}{(k-1)^2}.\]

\end{proof} 
We remark that calculations present in this lemma are at the heart of the analysis of RATQ.
Also, this lemma will be useful for other applications discussed in Section~\ref{s:appl}.  

\paragraph{Step 4: Completing the proof.}
Recall that the random matrix $R$ defined in~\eqref{e:R} is  used at the encoder of RATQ to randomly rotate the input vector. We observe that the rotated vector has subgaussian entries. 
\begin{lem}\label{l:concentration_a.s.}
For an $\R^d$-valued random variable $Y$ such that $\norm{Y}_2^2 \leq B^2 \text{ a.s.}$, all
coordinates of the rotated vector $RY$ are centered subgaussian random
variables with a variance factor of ${B^2}/{d}$, whereby
\[
P(|{RY}(j)|\geq M) \leq 2  e^{-dM^2/2B^2}, \quad \forall\,  j \in [d],
\]
where $RY(j)$ is the $j^{th}$ coordinate of the rotated vector.
\end{lem}
\noindent The proof uses similar calculations as~\cite{ailon2006approximate} and~\cite{suresh2017distributed}; it is deferred to the Appendix \ref{ap:R_subg}. 

Intuitively, the Lemma~\ref{l:concentration_a.s.} highlights the fact that overall energy $\norm{Y}_2^2$ in the input vector $Y$ is divided equally among all the coordinates after random rotation.

\paragraph{The worst-case second moment of RATQ.}
Note that by the description of RATQ which will be denoted by $Q_{{\tt at}, R}(RY),$ we have that
\[Q_{{\tt at}, R}(Y) =R^{-1} Q_{{\tt at}, I}(RY), \]
where $Q_{{\tt at}, I}$ is as defined in \eqref{e:Q_at_I}. Thus,
\begin{align}\label{e:Q_at}
Q_{{\tt at}, I}(RY)=[Q_{{\tt at}}(RY_{1,s})^{T}, \cdots,  Q_{{\tt
      at}}(RY_{\ceil{d/s},s})^{T}]^T,
      \end{align} where
the subvector $RY_{i,s}$ is given by
\begin{align*}\label{e:ithsubvector}
RY_{i, s}=[RY((i-1)s+1), \cdots, RY(\min\{is,d\})]^T. 
\end{align*}

To compute $\alpha(Q_{{\tt at}, R}(Y))$, we will first compute the
second moment for the output of RATQ. Specifically, using the fact $R$
is a unitary transform, we obtain
\eq{
  \E{\norm{Q_{{\tt at}, R}(Y)}_2^2}&=\E{\norm{R^{-1}Q_{{\tt at}, I}(RY)}_2^2}
  \\
  &=\E{\norm{Q_{{\tt at}, I}(RY)}_2^2}
  \\ 
  &=\sum_{j \in [d] } \E{(Q_{{\tt at}, I}(RY)(j))^2}
  \\
&=\sum_{i=1}^{\ceil{\frac{d}{s}}}\sum_{j = (i-1)s+1}^{\min\{is,d\}} \E{(Q_{{\tt at}, I}(RY)(j))^2}.
}

\new{We now observe that for our choice of $m$ and $h$ for RATQ given by \eqref{e:RATQ_levels},  we have 
\[
M_{h-1}^2\geq m(e^{*{\log_{e}^*(d/3)}})=(3B^2/d).(d/3)=B^2.
\]
Using this observation and noting that $R$ is a unitary matrix, 
we have that 
\[
\indic{\{\norm{RY}_{2} \leq M_{h-1}\}}=1 \text{ a.s.}.
\]
Also, noting that $|RY(j)| \leq \norm{RY}_{2}=\norm{Y}_2=B$  a.s., for all $j\in [d]$,
we get
\begin{equation}\label{e:RATQ_Unbiased}
\indic{\{|RY(j)| \leq M_{h-1}\}} =1 \text{ a.s.}, \forall j \in [d].
\end{equation}
Proceeding with these observations, we get
\eq{
  \E{\norm{Q_{{\tt at}, R}(Y)}_2^2}
& \leq \sum_{i=1}^{ \ceil{\frac{d}{s}}}\sum_{j =
    (i-1)s+1}^{\min\{is,d\}}
  \E{(Q_{{\tt at}, I}(RY)(j))^2 \indic{\{|RY(j)| \leq M_{h-1}\}}}\\
  & = \sum_{i=1}^{ \ceil{\frac{d}{s}}}\sum_{j =
    (i-1)s+1}^{\min\{is,d\}}
  \E{(Q_{{\tt at}, I}(RY)(j)-RY(j)+RY(j))^2 \indic{\{|RY(j)| \leq M_{h-1}\}}}\\
  &\leq  \sum_{i=1}^{ \ceil{\frac{d}{s}}}\sum_{j =
    (i-1)s+1}^{\min\{is,d\}}
  \E{\left((Q_{{\tt at}, I}(RY)(j)-RY(j))^2+ RY(j)^2\right) \indic{\{|RY(j)| \leq M_{h-1}\}}} ,
}
where the previous inequality uses the fact that, under the event $\{|RY(j)|\leq M_{h-1}\}$,
$Q_{{\tt at}, I}(RY)(j)$ is an unbiased estimate of $RY(j)$.} Namely,
\[
\E{Q_{{\tt at}, I}(RY)(j)\indic{\{|RY(j)|
\leq M_{h-1}\}}\mid R, Y} = \E{RY(j)\indic{\{|RY(j)|
\leq M_{h-1}\}}\mid R, Y}.
\] 
Therefore, noting that $R$ is a unitary matrix, we have  
\begin{align*}
\E{\norm{Q_{{\tt at}, R}(Y)}_2^2} \leq  \E{\sum_{j \in [d]}(RY(i)-Q_{{\tt at}, I}(RY)(j))^2  \indic{\{|RY(j)|\leq M_{h-1}\}}} +\E{\norm{Y}_2^2}.
\end{align*}
\new{
To bound the first term on the right-side we have the following lemma,
which will also be useful later in Section~\ref{s:distributed_mean}. 
\begin{lem}\label{l:useful}
For an $\R^d$-valued random variable $Y$ such that $\norm{Y}_2^2 \leq B^2$ {a.s.}. Then, for $m$ and $m_0$ set to be $3B^2/d$ and $(2B^2/d) \ln s$, respectively, we have that 
\eq{
  \E{\sum_{j \in [d]}(RY(i)-Q_{{\tt at}, I}(RY)(j))^2  \indic{\{|RY(j)|\leq M_{h-1}\}}} \leq B^2 \cdot  \frac{9+3\ln s}{(k-1)^2}.
}
\end{lem}
\begin{proof}
  By Lemma~\ref{l:concentration_a.s.} we have that all coordinates
  $RY(j)$ are centered subgaussian random variable with variance
  factor $B^2/d$.  Thus, the parameters $m$ and $m_0$ of RATQ set as
  in~\eqref{e:RATQ_levels}, and equal $3v$ and $2v \ln s$, respectively,
  where $v$ is the variance factor of each subgaussian coordinate. The result follows by invoking  Lemma~\ref{l:subg_mse}.
\end{proof}
}
Therefore, for any $Y$ such that $\norm{Y}_2^2 \leq B^2$, we have
\eq{
\E{\norm{Q_{{\tt at}, R}(Y)}_2^2} \leq  B^2 \cdot  \frac{9+3\ln s}{(k-1)^2} + B^2
,}
whereby
\[
\alpha_0(Q_{{\tt at}, R})\leq B\sqrt{ \frac{9 +
3 \ln s}{(k-1)^2} + 1}.
\]

\paragraph{The worst-case bias of RATQ.} By \eqref{e:RATQ_Unbiased} we have that the input always remains
in the dynamic-range of the quantizer, resulting in unbiased quantized
values. In other words,  $\beta_0(Q_{{\tt at, R}})=0$.

\subsection{Proof of Lemma \ref{l:AGUQ_alpha_beta}}
\label{s:AGUQ}
We first note AGUQ is used to quantize a scalar $Y$.
It follows from the description of the quantizer that
\begin{equation}\label{e:MeanAGUQ}
\mathbbm{1}_{\{|{Y}| \leq M_{g, h_g-1}\}} \E{Q_{\tt a}(Y) |Y} = \mathbbm{1}_{\{|{Y}| \leq M_{g, h_g-1}\}}   Y,
\end{equation}
and that\footnote{Once again, this follows from our convention that the outflow symbol is evaluated to $0$.}
 \begin{equation}\label{e:MeanAGUQSec}
 \mathbbm{1}_{\{|{Y}| > M_{g, h_g-1}\}} Q_{\tt a}(Y)  = 0.
 \end{equation} 
 
\paragraph{The worst-case second moment of AGUQ.}
Towards evaluating $\alpha(Q_{{\tt a}})$ for AGUQ,  for any $Y \in \R$ we have
\begin{align}
\nonumber
\E{Q_{\tt a}(Y)^2} &= \E{Q_{\tt a}(Y)^2\mathbbm{1}_{\{|Y| \leq M_{g, h_g-1}\}}}+\E{Q_{\tt a}(Y)^2\mathbbm{1}_{\{|Y| > M_{g, h-1}\}}}\\
\nonumber
&= \E{(Q_{\tt a}(Y)-Y+Y)^2\mathbbm{1}_{\{|Y| \leq M_{g, h_g-1}\}}}+\E{Q_{\tt a}(Y)^2\mathbbm{1}_{\{|Y| > M_{g, h_g-1}\}}}
\\ \nonumber
&= \E{(Q_{\tt a}(Y)-Y)^2\mathbbm{1}_{\{|Y| \leq M_{g, h_g-1}\}}}+\E{Y^2\mathbbm{1}_{\{|Y| \leq M_{g, h_g-1}\}}},
\nonumber
\end{align}
where the last identity uses
\eqref{e:MeanAGUQSec}, and the fact that $\E{(Q_{\tt
    a}(Y)-Y)Y\mathbbm{1}_{\{|Y| \leq M_{g, h-1}\}}|Y}=0,$ which
follows from \eqref{e:MeanAGUQ}.  From Lemma \ref{l:sup3} it follows
that
\eq{ \E{(Q_{{\tt a}}(Y)-Y)^2\mathbbm{1}_{\{|Y| \leq M_{g,
        h-1}\}}} \leq \frac{1}{4(k_g-1)^2}\left(B^2+ \sum_{j=1}^{h-1}
  M_j^2P\left(|Y| > M_{g, j-1}\right)\right).
}
By Markov's
inequality we get that for any random variable $Y$ with $\E{Y^2}\leq B^2$, we have
$P(|Y| > M_{g, j-1})\leq B^2/M^2_{g, j-1},$ which further leads to
\eq{
  \E{(Q_{{\tt a}}(Y)-Y)_2^2\mathbbm{1}_{\{|Y| \leq
    M_{g, h-1}\}}}  &\leq \frac{B^2}{4(k_g-1)^2}+ \sum_{j=1}^{h_g-1}
 \frac{M_{g,j}^2 }{4(k_g-1)^2} \frac{B^2}{M_{g,j-1}^2}\\
 &=  \frac{B^2}{4(k_g-1)^2}+ \frac{a_g(h_g-1)B^2}{4(k_g-1)^2}.
}
Therefore, we have
\eq{
\E{Q_{\tt a}(Y)^2} \leq \frac{B^2}{4(k_g-1)^2}+ \frac{a_g(h_g-1)B^2}{4(k_g-1)^2} +\E{Y^2\mathbbm{1}_{\{|Y| \leq M_{g, h-1}\}}}.
}
The result follows upon taking the supremum of the left-side over all
random variables $Y$ with $\E{Y^2} \leq B^2$.

\paragraph{The worst-case bias of AGUQ.}
Towards evaluating $\beta(Q_{\tt a})$, we note first using Jensen's inequality that
\eq{
  \big|\E{Q_{\tt a}(Y)-Y}\big| 
  &\leq \E{\big|\E{Q_{\tt a}(Y)-Y|Y}\big|}.
}
Then, for $Y$ with $\E{Y^2}\leq B^2$, using \eqref{e:MeanAGUQ} and Markov's inequality, we get
\begin{align}\label{e:conditionBeta_AGUQ}
\nonumber \E{|\E{Q_{\tt a}(Y)-Y|Y}|}
&= \E{|Y|\mathbbm{1}_{\{|Y| \geq M_{g, h-1}\}}}\\
\nonumber &\leq \sqrt{\E{Y^2}P(|Y| \geq M_{g, h-1})}\\
&\leq \frac{B^2}{M_{g, h-1}}.
\end{align}
Therefore, for any $Y$ with $\E{Y^2}\leq B^2$, we have 
\eq{
  \big|\E{Q_{\tt a}(Y)-Y}\big|  \leq
\sup_{Y\geq 0 \text{a.s.}: \E{Y^2}\leq B^2}\E{\big|\E{Q_{\tt a}(Y)-Y|Y}\big|}\leq
  \frac{B^2}{M_{g, h-1}}.
}
The result follows upon taking the supremum of left-side over all
random variables $Y$ with $\E{Y^2} \leq B^2$.  \qed

\subsection{Proof of Theorems \ref{t:lb_1} and \ref{t:lb_2}}\label{s:lbproof}
Before we proceed with our lower bounds, we will set up some
notation. We consider quantizers of the form
\[
Q(Y)=Q_g(\norm{Y}_2)Q_s(Y/\norm{Y}_2).
\] 
Let $W(\cdot |y)$, $W_g(\cdot| y)$, and $W_s(\cdot |y)$, respectively,
denote the distribution of the output of quantizers $Q(y)$, $Q_g(y),$ and $Q_s(y)$.
We prove a general lower bound for a quantizer satisfying Assumptions
1-3 in Section~\ref{s:ug} in terms of the precision $r$; Theorems~\ref{t:lb_1} and~\ref{t:lb_2}
are obtained as corollaries of this
general lower bound. 
\begin{thm}\label{t:lb}
  Suppose that $\X$ contains the set $\{x\in \R^d: \norm{x}_2
  \leq D/2\}$.  
  Consider a gain-shape quantizer $Q$ of precision $r$ satisfying the  Assumptions 1-3 in Section~\ref{s:ug}. Then, there exists an oracle $(f, O)\in \oO$ such that for any optimization protocol $\pi$ using $T$
  iterations, we have 
  \[
  \ep(f, \pi^{QO}) \geq 
\frac{DB}{2\sqrt{2}}\min\bigg\{ \frac{1 }{2^r} ,\frac 1{4\cdot2^{r/3}T^{1/3}},\frac{1}{2(2T)^{1/3}}  \bigg\}.
\]
\end{thm}
\begin{proof}
Consider the function $f_\alpha: \R^d \to \R$, $\alpha\in\{-1,1\}$ given by
  \[
f_\alpha(x) \eqdef \delta \frac{B}{\sqrt{2}}|x(1)- \alpha D/2|, \quad \alpha \in \{-1, 1\}.
\]
Note that the functions $f_1$ and $f_{-1}$ are convex and depend only on the first
coordinate of $x$. Further, for $x\in\X$, the gradient of $f_\alpha$
is $\delta \alpha B e_1/\sqrt{2}$, where $e_1$ is the vector $[1, 0, 0, \dots, 0]^T$.  We consider oracles $O_\alpha$, $\alpha\in\{-1,1\}$, that produce noisy gradient updates with distribution 
\eq{
 P_{\alpha}\left(\frac{B}{\sqrt{2}}e_1\right) =
   \frac{1-\delta^2}{2} , \quad  P_{\alpha}\left(\frac{-B}{\sqrt{2}}e_1\right) =
   \frac{1-\delta^2}{2} ,\quad
P_{\alpha}\bigg(\frac{\alpha B}{\sqrt{2}\delta}e_1\bigg) =
   \delta^2.
}  
 It is easy to check that the oracle outputs satisfy ~\eqref{e:asmp_unbiasedness}
and~\eqref{e:asmp_L2_bound} described in Section \ref{s:problemsetup}. That is, the output of $O_{\alpha}$ is an unbiased estimate of the subgradient of $f_{\alpha}$, and the expected Euclidean norm square of the oracle output is bounded by $B^2$. 

We now take recourse to the standard reduction of optimization to
hypothesis testing: To estimate the optimal value of $f_{1}$ and
$f_{-1}$ to an accuracy $\delta$, the optimization protocol must
determine if the oracle outputs are generated by $P_{1}$ or
$P_{-1}$.  However in order to distinguish between $P_{1}$ or
$P_{-1}$, the optimization protocol only has access to the quantized oracle outputs.  Specifically, the protocol sees the samples from $Q(Y)$ at every time step, where $Y$ has distribution either $P_1$ or $P_{-1}$.

Denoting by $P_\alpha W$ the distribution
of the output  $ Q(Y)$ when the input $Y$ is
generated from $P_\alpha$,  
 we have from the standard reduction (see, for instance, \cite[Theorem 5.2.4]{duchi2017introductory}) that
 \[
\max_{\alpha\in\{-1,1\}}\ep(f, \pi^{QO}) \geq 
\frac{DB}{2\sqrt{2}}\delta\bigg( 1- \sqrt{\frac{T}{2}\chi^2(P_1W,P_{-1}W}\bigg), 
\]
where $\displaystyle{\chi^2(P, Q)=\sum_{x}(P(x)-Q(x))^2/Q(x)}$ denotes the chi-squared divergence between $P$ and $Q$.

Note that Assumption 2 on the structure of the quantizer implies that when $M< B/{\delta\sqrt{2}}$, the distributions $P_1W$ and $P_{-1}W$ are the same. It follows that for every $\delta<\min\{\sqrt{B^2/2M^2}, 1\}$, the
left-side of the previous inequality exceeds $(DB/2\sqrt{2})\delta$, whereby
\begin{align}
\max_{\alpha\in\{-1,1\}}\ep(f, \pi^{QO})
\geq \frac{DB}{2\sqrt{2}}\min \bigg\{\frac{B}{\sqrt{2}M}, 1\bigg\}.
\label{e:bound1}
\end{align}

Next, we consider the following modification of the previous
construction in the case when ${B}/\sqrt{2}< m$:  
 \eq{ 
  P_{\alpha}\left(\frac{B}{\sqrt{2}}e_1\right) =
   \frac{1-\delta^{1+y}}{2} , \quad  P_{\alpha}\left(\frac{-B}{\sqrt{2}}e_1\right) =
   \frac{1-\delta^{1+y}}{2} ,\quad
P_{\alpha}\bigg(\frac{\alpha B}{\sqrt{2}\delta^y}e_1\bigg) =
   \delta^{1+y}.
}
for $y\in [0,1]$. Once again, the oracle outputs satisfy ~\eqref{e:asmp_unbiasedness}
and~\eqref{e:asmp_L2_bound} described in Section \ref{s:problemsetup}. In this case, the  vector $Y \sim P_{\alpha}$ has entries with $\ell_2$ norm at the most $ B/(\sqrt{2}\delta^y)$. We set $y$ such that this value is less than $m$ and $\chi^2(P_1W,P_{-1}W)$ is minimized.  Note that if $B/(\delta^y\sqrt{2d})< m$, then ${\tt supp}(Q_{g}(\norm{a})) \subseteq \{0, m\}$ for all the $a$'s in the support of $P_1$ or $P_{-1}$. 
      
      For all $z \neq 0$, $z \in  {\tt supp}(Q(a))$, when $a$ is in the support of $P_1$ or $P_{-1}$, we have
\eq{W\bigg(z\bigg|a \bigg)=W_g\bigg(m|\norm{a}_2\bigg) W_s\bigg(\frac{z}{m}|a/\norm{a}_2\bigg).}
Therefore,
 \eq{
 P_{1}W(z)-P_{-1}W(z) &=\delta^{1+y}  W_g\bigg(m|\frac{B}{\sqrt{2}\delta^y}\bigg) \left( W_s\bigg(\frac{z}{m}|e_1\bigg)   - W_s\bigg(\frac{z}{m}|-e_1\bigg)\right)
 }
\eq{
P_{-1}W(z) &\geq \frac{1-\delta^{1+y}}{2} W_g\bigg(m|\frac{ B}{\sqrt{2}}\bigg) W_s\bigg(\frac{z}{m}|e_1\bigg) + \frac{1-\delta^{1+y}}{2}  W_g\bigg(m|\frac{ B}{\sqrt{2}}\bigg) W_s\bigg(\frac{z}{m}|-e_1\bigg). }
Using the preceding two inequalities 
\eq{
\frac{(P_{1}W(z)-P_{-1}W(z))^2}{P_{-1}W(z)} &\leq  \frac{\delta^{2+2y}  W_g\bigg(m|\frac{B}{\sqrt{2}\delta^y}\bigg)^2 \left( W_s\bigg(\frac{z}{m}|e_1\bigg)   - W_s\bigg(\frac{z}{m}|-e_1\bigg)\right)^2}{\frac{1-\delta^{1+y}}{2} W_g\bigg(m|\frac{ B}{\sqrt{2}}\bigg) W_s\bigg(\frac{z}{m}|e_1\bigg) + \frac{1-\delta^{1+y}}{2}  W_g\bigg(m|\frac{ B}{\sqrt{2}}\bigg) W_s\bigg(\frac{z}{m}|-e_1\bigg)}\\
&\leq \frac{2\delta^{2+2y}}{1-\delta^{1+y}} \cdot \frac{ W_g\bigg(m|\frac{B}{\sqrt{2}\delta^y}\bigg)^2 \left(W_s\bigg(\frac{z}{m}|e_1\bigg)   + W_s\bigg(\frac{z}{m}|-e_1\bigg)\right)}{W_g\bigg(m|\frac{ B}{\sqrt{2}}\bigg)} \\
&\leq  \frac{2\delta^{2+y}}{1-\delta^{1+y}} \cdot  W_g\bigg(m|\frac{B}{\sqrt{2}\delta^y}\bigg) \left(W_s\bigg(\frac{z}{m}|e_1\bigg)   + W_s\bigg(\frac{z}{m}|-e_1\bigg)\right)\\
&\leq  \frac{2\delta^{2+y}}{1-\delta^{1+y}} \cdot \left(W_s\bigg(\frac{z}{m}|e_1\bigg)   + W_s\bigg(\frac{z}{m}|-e_1\bigg)\right),
}
\noindent where the third inequality uses Assumption 3b for the quantizer in Section \ref{s:ug}, i.e., it uses
\[
\frac{ W_g\bigg(m|\frac{B}{\sqrt{2}\delta^y}\bigg)}{W_g\bigg(m|\frac{ B}{\sqrt{2}}\bigg)} \leq \delta^{-y}
.
\]
For all $z = 0$, $z \in  {\tt supp}(Q(a))$, when $a$ is in the support of $P_1$ or $P_{-1}$, we have
\eq{W\bigg(0\bigg|a \bigg)=W_g\bigg(0|\norm{a}_2\bigg) + W_g\bigg(m|\norm{a}_2\bigg) W_s\bigg(0|a/\norm{a}_2\bigg).}

Therefore, by similar calculations for $z\neq 0$, we have
\eq{
\frac{(P_{1}W(0)-P_{-1}W(0))^2}{P_{-1}W(0)} &\leq \frac{\delta^{2+2y} W_g\bigg(m|\frac{B}{\sqrt{2}\delta^y}\bigg)^2 \left( W_s\bigg(0|e_1\bigg)   + W_s\bigg(0|-e_1\bigg)\right)^2}{(\frac{1-\delta^{1+y}}{2}) W_g\bigg(m|\frac{ B}{\sqrt{2}}\bigg) W_s\bigg(0|e_1\bigg) + (\frac{1-\delta^{1+y}}{2})  W_g\bigg(m|\frac{ B}{\sqrt{2}}\bigg) W_s\bigg(0|-e_1\bigg)}\\
&\leq  \frac{2\delta^{2+y}}{1-\delta^{1+y}} \left( W_s\bigg(0|e_1\bigg)   + W_s\bigg(0|-e_1\bigg)\right)
.}

In conclusion,
\eq{  
\chi^2(P_1W, P_{-1}W) \leq \frac{4\delta^{2+y}}{1-\delta^{1+y}}.
}
Now, if $\delta < 1/2,$ we have
\eq{  
\chi^2(P_1W, P_{-1}W) \leq 8\delta^{2+y}.
}
 Upon setting $\delta=(16 T)^{-1/(2+y)}$, which satisfies $\delta<1/2$ for all $T$, 
we get
\begin{align}
\max_{\alpha\in\{-1,1\}}\ep(f, \pi^{QO})\geq \frac{DB}{2\sqrt{2}}\delta( 1- \sqrt{4T\delta^{2+y}})
= \frac{DB}{4\sqrt{2}}\left(\frac{1}{16T}\right)^{\frac 1 {2+y}}.
\label{e:bound2}
\end{align}
But we can only set $\delta$ to this value if
\begin{align}
\frac{B}{\sqrt{2}}\cdot \left({16 T}\right)^{\frac y {2+y}}< m .
\label{e:y_constraint}
\end{align}
Thus, for each $y$ such that~\eqref{e:y_constraint} holds, we get~\eqref{e:bound2}.
Taking the the supremum of RHS in \eqref{e:bound2} over all  $y\in[0,1]$ such that \eqref{e:y_constraint} holds, we obtain whenever $B/\sqrt{2} \leq m,$
\begin{align*}
\max_{\alpha\in\{-1,1\}}\ep(f, \pi^{QO})\geq 
\frac{DB}{2\sqrt{2}} \cdot \min\left\{ \frac{1}{8}\sqrt{\frac{m\sqrt{2}}{BT}} , \frac{1}{2(2T)^{1/3}} \right\},
\end{align*}
where we use the following lemma proved in Appendix~\ref{app:elemineq}.
\begin{lem}\label{l:elemineq}
For $a, c > 0,$ and $b > 1$. 
\[\sup_{y \in [0,1]: a(b)^{y/(2+y)}< c.} a\left(\frac{1}{b}\right)^{\frac{1}{2+y}}= \min\left\{ \sqrt{\frac{ca}{b}}, \frac{a}{b^\frac{1}{3}}\right\}\]
\end{lem}
Upon combining this bound with~\eqref{e:bound1}, we obtain 
\[
\sup_{(f,O)\in \oO}\varepsilon( f_\alpha, \pi^{QO})\geq 
\frac{DB}{2\sqrt{2}}\max\left\{\min\left\{\frac{c m}{M}, 1\right\}, \min\left\{\frac{1}{8}\sqrt{\frac{1}{cT}}, \frac{1}{2(2T)^{1/3}}\right \}\mathbbm{1}_{\{c<1\}}\right\},
\]
where $c=B/(m\sqrt{2})$. By making cases $1\leq c$, $ \frac{1}{8(2T)^{1/3}} \leq c<1$, and $c < \frac{1}{8(2T)^{1/3}}$, and using the fact that for $a,b \geq 0$, $\max\{a,b\}\geq  a^{1/3}b^{2/3}$ in the second case, we get
 \[
\sup_{(f,O)\in \oO}\varepsilon( f_\alpha, \pi^{QO})\geq 
\frac{DB}{2\sqrt{2}}\min\bigg\{1, \frac{1 }{(M/m)} ,\frac{1}{4 (M/m)^{1/3}T^{1/3}},\frac{1}{2(2T)^{1/3}}  \bigg\}.
\]
By Assumption 3 in Section \ref{s:ug}, we know that $\frac{M}{m}\leq 2^r$. Therefore,
 \[
\sup_{(f,O)\in \oO}\varepsilon( f_\alpha, \pi^{QO})\geq 
\frac{DB}{2\sqrt{2}}\min\bigg\{ \frac{1 }{2^r} ,\frac{1}{4 (2)^{r/3}T^{1/3}},\frac{1}{2(2T)^{1/3}}  \bigg\}.
\]
 \end{proof}

Theorem~\ref{t:lb_2} follows as an immediate corollary;
Theorem~\ref{t:lb_1}, too, is obtained by noting that
\[
\sup_{(f,O)\in
  \oO}\varepsilon( f_\alpha, \pi^{QO})< \frac{3DB}{\sqrt{T}}
\]
holds only if  $\sqrt{T} < 2^r$.

%% file: applications.tex
\section{Other applications of our quantizers}\label{s:appl}
In this section we will discuss applications of our quantizers for other related problems. 
\subsection{RATQ and  distributed mean estimation}\label{s:distributed_mean}
We now discuss the performance of RATQ for distributed mean estimation
with limited communication, considered in
~\cite{suresh2017distributed}. To state our results formally, we
describe the setting. Consider $n$ vectors $\{x_i\}_{i=1}^{n}$ with
each $x_i$ in $\R^d$ and vector $x_i$ available to client $i$.  Each
client communicates to a fusion center using $r$ bits to enable the
center to compute the sample mean
\[
\bar{x}=\frac{1}{n}\sum_{i=1}^n x_i.
\]
\newest{We seek to design quantizers which can be used to express
  $x_i$s using $r$ bits each and yield a high accuracy estimate of
  $\bar{x}$ at the center.}


Specifically, consider two cases for the quantization schemes: The
case of {\em fixed-length} codes where client $i$ uses a randomized
encoding mapping $e_i: \R^d\to \{0,1\}^r$, $i\in [n]$, and the center
uses a decoding mapping $d:\{0,1\}^{nr}\to \R^d$; and the case of {\em
  variable-length} codes where the encoder mappings are $e_i:\R^d\to
\{0,1\}^\ast$ and must satisfy $\E{|e_i(x_i)|}\leq r$ for each $x_i\in
\R^d$ and each $i\in[n]$, where $|b|$ denotes the length of a binary
vector $b$. A distributed mean estimation code $\pi$ is specified by
the encoder mappings $e_i$ and the decoder mapping $d$. We emphasize
that the encoder mappings and the decoder mapping are allowed to be
randomized using shared randomness; namely, we allow public-coin
simultaneous message passing protocols.  We denote the set of all
fixed-length codes by $\Pi(r)$ and the set of all variable-length
codes by $\Pi_\ast(r)$.
  
We measure the performance of a code $\pi$ by the mean square error
(MSE) between $\bar{x}$ and $\hat{\bar{x}}=d(e_1(x_1), ...,
e_n(x_n)))$, for a fixed input vector $x=(x_1, ...,x_n)$, given by
\[
\ep(\pi,x)=\E{\norm{\hat{\bar{x}}-\bar{x}}_{2} ^2}.
\] 
We consider a minmax setting where we allow the input vectors $x=(x_1,
..., x_n)$ to be chosen arbitrarily from the unit Euclidean ball
$B^d$. That is, consider the worst-case MSE over all vectors in $B^d$
given by
\[
\ep(\pi, B^d) = \max_{ x_i \in B^d, \forall i \in [n] }\ep(\pi,x)
\]
The minimum error attained by fixed-length codes is given by
\[
\ep(\Pi(r), B^d) = \min_{\pi \in \Pi(r)}\ep(\pi, B^d),
\]
and that by variable-length codes is given by
\[
\ep(\Pi_\ast(r), B^d) = \min_{\pi \in \Pi_\ast(r)}\ep(\pi, B^d),
\]

The following lower bound is from~\cite[Theorem
  5]{suresh2017distributed} (where it was shown using a construction
from~\cite{ZDJW:13}).
\begin{thm}[~\cite{suresh2017distributed, ZDJW:13}]
There exists a constant $t$ such that for every $r \leq ndt/4$ and $n
\geq 4/t$, we have
\[
\ep(\Pi_{\ast}(r), B^d) \geq \ep(\Pi_\ast(r), B^d)\geq
\frac{t}{4}\min\{1,\frac{d}{r}\}.
\]
\end{thm}
As a corollary, we have the following alternative form of the same
lower bound.
\begin{cor}
For $\ep(\Pi_\ast(r), B^d)=O(1/n)$, we must have $r$ to be
$\Omega(nd)$.
\end{cor}

The protocol $\pi_{srk}$ proposed in~\cite{suresh2017distributed} for
this problem achieves $\ep(\pi_{srk}, B^d)=O(1/n)$ with $r = \Omega(n
d \log\log (d))$. This scheme uses a quantizer that randomly rotates a
input vector, similar to RATQ, before quantizing it uniformly.  A
simpler quantizer similar to CUQ with a variable-length entropic
compression code, denoted by $\pi_{svk}$, achieves $\ep(\pi_{svk},
B^d)=O(1/n)$ with $r = \Omega( n d )$. This establishes the orderwise
optimality of $\pi_{svk}$. Thus, prior to our work, the best known
fixed-length scheme for distributed mean estimation was $\pi_{srk}$
which was off from the optimal performance attained by a
variable-length code by a factor of $\log \log d$.


We now consider performance of a protocol $\pi_{RATQ}$ in which RATQ
is employed by all the clients, and the center declares the average of
the quantized values as its mean estimate. First, we have the
following lemma which describes the mean square performance of RATQ.
\begin{lem}\label{l:MSE_bound}
Let $Q_{{\tt at}, R}$ be the quantizer RATQ with parameters $m$,
$m_0$, and $h$ as in~\eqref{e:RATQ_unit_levels}. Then, for every
$\R^d$-valued random variable $Y$ such that $\norm{Y}_2^2 \leq 1$
{a.s.}, we have
 \[\E{ \norm{Q_{{\tt at}, R}(Y) -Y }_2^2 } \leq  \frac{9+3 \ln s}{(k-1)^2}.\]
 \end{lem}
\noindent See Appendix~\ref{ap:dist_mean_1} for the proof.

\begin{thm}\label{t:RATQ_distributed_mean}
Let $\pi_{RATQ}$ be a protocol for distributed mean estimation in
which RATQ with parameters $m$, $h$ as in \eqref{e:RATQ_unit_levels}
and $k$, $s$ as in \eqref{e:RATQ_bits} is employed by all the
clients. Then, $\ep(\pi_{RATQ}, B^d)=O(1/n)$ with a total precision of
\[
r^{\prime}= n \left( d(1+\Delta_1)+\Delta_2 \right),
\]
where $\Delta_1$ and $\Delta_2$ are as in
Corollary~\ref{c:PSGD_RATQ_0}. This further yields $\ep(\Pi(r),
B^d)=O(1/n)$ for every $r \geq r^{\prime}$.
\end{thm}
\noindent See Appendix \ref{ap:dist_mean_2} for the proof.

Thus, RATQ enjoys the fixed length structure of $\pi_{srk}$, while
being only $O(\log \log\log \ln^*(d/3))$ away from the expected length
of $\pi_{svk}$.

\newest{
\subsection{ATUQ and Gaussian rate distortion}\label{s:gaussian_rate}
In this final section, we consider the classic Gaussian
rate-distortion problem. We first describe this problem.

Consider a random vector $X=[X(1),\cdots, X(d)]^T$ with 
$iid$ components $X(1), \cdots, X(d)$ generated from
a zero-mean Gaussian distribution with variance $\sigma^2$. For pair $(R,D)$ of nonnegative
numbers
is an {\em achiveable} rate-distortion pair if we can find a quantizer $Q_d$
of precision $dR$ and with mean square error $\E{\norm{X -
    Q_d(X)}_2^2}\leq dD$.
For $D>0$, denote by $R(D)$ the infimum over all
$R$ such that $(R,D)$ constitute an achievable rate-distortion pair
for all $d$ sufficiently large. A well-known result in information theory
characterizes $R(D)$ as follows  ($cf.$~\cite{CovTho06}):
\eq{
R(D)=
 \begin{cases}
\frac{1}{2} \log \frac{\sigma^2}{D} \quad &\text{ if } D \leq
\sigma^2,  \\ 
0 &\text{ if } D > \sigma^2.    
\end{cases}}
The function $R(D)$ is called the {\em Gaussian rate-distortion
  function}.

Over the years, several constructions using error correcting codes and
lattices have evolved that attain the rate-distortion function,
asymptotically for large $d$. In this section, we show that a slight
variant of ATUQ, too, attains a rate very close to the Gaussian
rate-distortion function, when applied to Gaussian random vectors.

Specifically, consider the quantizer $Q_{{\tt at}, I}$ described
earlier in~\eqref{e:Q_at_I}.
Recall that  $Q_{{\tt at}, I}$ can be
described by algorithm \ref{a:E_RATQ} and \ref{a:D_RATQ} with random
matrix $R$ replaced with $I$. That is, we
divide the input vector in $\ceil{d/s}$ and employ ATUQ to quantize
them. In fact, we will apply this quantizer not only to a Gaussian
random vector, but any random vector with subgaussian
components; the components need not even be independent.
Thus, we show that our quantizer is almost optimal {\em
  universally} for all subgaussian random vectors. 

We set the parameters $m$, $m_0$, $h$, $s$, and $\log(k+1)$ of $Q_{at, I}$ as follows:
\begin{align}\label{e:Param}
\nonumber 
m = 3v, \quad m_0= 2v \ln s, \quad \log h = \ceil{\log\left(1+\ln^*\left(\frac{4\ln(8\sqrt{2}v/D)}{3}\right)\right)},\\
s = \min \{ \log h, d\}, \text{ and } \log (k+1) =
\ceil{\log\left(2+\sqrt{\frac{18v+6v\ln s}{D}}\right)}. 
\end{align}
 \begin{thm}\label{t:gauss_rd}
   Consider a random vector $X$ taking values in $\R^d$ and with 
   components $X_i$, $1\leq i \leq d$ such that each $X_i$ is a
   centered subgaussian random variables with a variance factor $v$. 
 Let $Q_d$ be the $d$-dimensional $Q_{at, I}$ with parameters as in
 \eqref{e:Param}. Then, for $d\geq \log h$ and $D<v/4$, $Q_d$ gets the mean square error less than $dD$ using
 rate $R$ satisfying
 \[
 R \leq \frac{1}{2}\log\frac{v}{D}+ O\left(\log\log\log
 \log^*\log\left (\frac{v}{D}\right)\right).
 \]
 \end{thm}
 \noindent We provide the proof in Appendix~\ref{ap:gauss_rd}. We
 remark that the additional term is a small constant for reasonable
 values of the parameters $v$ and $D$. Note that our proposed
 quantizer just uses uniform quantizers with different dynamic ranges,
 and yet is almost universally rate optimal.   
}

%% file: appendix.tex
\section{ Analysis of PSGD with quantized subgradients: Proof of Theorem \ref{t:basic_convergence}}\label{ap:QPSGD}

We proceed as in the standard proof of convergence (see, for
instance,~\cite{bubeck2015convex}): Denoting by $\Gamma_\X(x)$ the
projection of $x$ on the set $\X$ (in the Euclidean norm), the error
at time $t$ can be bounded as
\eq{
  \norm{x_t-x^*}_{2}^{2} &= \norm{\Gamma_{\X}\big(x_{t-1}-\eta
    Q(\hat{g}(x_{t-1}))\big)-x^*}_{2}^{2} \\
    & \leq \norm{\big(x_{t-1}-\eta
    Q(\hat{g}(x_{t-1}))\big)-x^*}_{2}^{2} \\ &
    = \norm{x_{t-1}-x^*}_{2}^{2} + \norm{\eta
    Q(\hat{g}(x_{t-1}))}_{2}^{2}- 2 \eta
    (x_{t-1}-x^*)^{T}Q(\hat{g}(x_{t-1}))\\
    &=\norm{x_{t-1}-x^*}_{2}^{2} + \norm{\eta
    Q(\hat{g}(x_{t-1}))}_{2}^{2}- 2 \eta
    (x_{t-1}-x^*)^{T}\big(Q(\hat{g}(x_{t-1}))
    - \hat{g}(x_{t-1})\big) \\& \quad - 2 \eta
    (x_{t-1}-x^*)^{T}\hat{g}(x_{t-1}), } where the first inequality is
    a well known property of the projection operator $\Gamma$ (see,
    for instance, Lemma 3.1,
\cite{bubeck2015convex}).  By rearranging the terms, we have
\eq{
  2\eta (x_{t-1}-x^*)^{T}\hat{g}(x_{t-1})
  &\leq \norm{x_{t-1}-x^*}_{2}^{2} - \norm{x_{t}-x^*}_{2}^{2}
  + \norm{\eta Q(\hat{g}(x_{t-1}))}_{2}^{2} \\ & \hspace{1cm} -2 \eta
  (x_{t-1}-x^*)^{T}\left(Q(\hat{g}(x_{t-1}))
  - \hat{g}(x_{t-1})\right).  } Also, since $\E{\hat{g}(x_{t-1})|
  x_{t-1}}$ is a subgradient at $x_{t-1}$ for the convex function $f$,
  upon taking expectation we get
\[
\E{f(x_{t-1})-f(x^*)}\leq
\E{(x_{t-1}-x^*)^{T}\E{\hat{g}(x_{t-1})| x_{t-1}}}, 
\]
which with the previous bound yields
\eq{
  2\eta\E{f(x_{t-1})-f(x^*)} & \leq \E{\norm{x_{t-1}-x^*}_{2}^{2} }
  - \E{\norm{x_{t}-x^*}_{2}^{2} }+ \eta^2 \E{ \norm{
  Q(\hat{g}(x_{t-1}))}_{2}^{2}} \\ & \hspace{1cm} -
  2 \eta \E{(x_{t-1}-x^*)^{T}\big(Q(\hat{g}(x_{t-1}))
  - \hat{g}(x_{t-1})\big)}.  } Next, by the Cauchy-Schwarz inequality
  and the assumption in (1), the third term on the right-side above
  can be bounded further to obtain
\eq{2\eta\E{f(x_{t-1})-f(x^*)} & \leq \E{\norm{x_{t-1}-x^*}_{2}^{2} }
  - \E{\norm{x_{t}-x^*}_{2}^{2} }+ \eta^2 \E{ \norm{
      Q(\hat{g}(x_{t-1}))}_{2}^{2}} \\& \hspace{2 cm} + 2 \eta\cdot
      D\cdot \E{\norm{\E{Q(\hat{g}(x_{t-1}))
      - \hat{g}(x_{t-1})|x_{t-1} }}_2} .  } Finally, we note that, by
      the definition of $\alpha$ and $\beta$, for $L_2$-bounded
      oracles we have
\begin{align*}
\E{ \norm{
    Q(\hat{g}(x_{t-1}))}_{2}^{2}} &\leq \alpha(Q)^2,
\\
\norm{\E{Q(\hat{g}(x_{t-1})) - \hat{g}(x_{t-1})|x_{t-1} }}_2&\leq \beta(Q),
\end{align*}

which gives
\eq{2\eta\E{f(x_{t-1})-f(x^*)} & \leq \E{\norm{x_{t-1}-x^*}_{2}^{2} }
  - \E{\norm{x_{t}-x^*}_{2}^{2} }+ \eta^2 \alpha(Q)^2+ 2 \eta
    D \beta(Q). }
 
Therefore, by summing from $t=2$ to $T+1$, dividing by $T$, and using
assumption that the domain $\X$ has diameter at the most $D$, we have
\eq{
  2\eta\E{f(\bar{x}_{T})-f(x^*)}
  & \leq \frac{D^2}{T}+ \eta^2 \alpha(Q)^2+ 2 \eta D \beta(Q). }
 
The first statement of Theorem~\ref{t:basic_convergence} follows upon
dividing by $\eta$ and setting the value of $\eta$ as in the
statement. The second statement holds in a similar manner by replacing
$\alpha$ and $\beta$ with $\alpha_0$ and $\beta_0$, respectively.
\qed




\section{Remaining proofs for the main results}

\subsection{Analysis of CUQ: Proof of Lemmas \ref{l:sup} and \ref{l:sup0m} }\label{ap:CUQ}

\paragraph{Proof of Lemma \ref{l:sup}:}

Denoting by $\B_{j,\ell}$ the event $\big\{ {Y}(j) \in [B_{M,k}(\ell),
  B_{M,k}(\ell+1))\big\}$, we get
\eq{
&\lefteqn{\E{\sum_{j\in[d]}\big(Q_{{\tt
u}}(Y)(j)-{Y(j)}\big)^2\mathbbm{1}_{\{{|Y(j)|}\leq M\}}\mid Y}}
\\
&= \sum_{j \in [d]} \sum_{\ell =0}^{k-1}\E{\big(Q_{{\tt
       u}}(Y)(j)-{Y}(j)\big)^2 \mathbbm{1}_{\B_{j,\ell}}\mid
       {Y}} \mathbbm{1}_{\{{|Y(j)|}\leq M\}}.  }  For the summand on
       the right-side, we obtain
\begin{align}\label{e:residual}
 \nonumber &\lefteqn{\E{\big(Q_{{\tt
u}}(Y)(j)-{Y}(j)\big)^2 \mathbbm{1}_{\B_{j,\ell}}\mid
{Y}}} \\ \nonumber &= \left(( B_{M,k}(\ell+1)-{Y}(j))^2 \frac{{Y}(j)
-B_{M,k}(\ell)
}{B_{M,k}(\ell+1)-B_{M,k}(\ell)}\right) \mathbbm{1}_{\B_{j,\ell}}
\\ \nonumber
&\hspace{2cm} + \left(( B_{M,k}(\ell)-{Y}(j))^2\frac{ B_{M,k}(\ell+1)-
{Y}(j)}{B_{M,k}(\ell+1)-B_{M,k}(\ell)} \right) \mathbbm{1}_{\B_{j,\ell}}
\\ \nonumber
&= \left( B_{M,k}(\ell+1)-{Y}(j)) ({Y}(j) -
B_{M,k}(\ell)\right)\mathbbm{1}_{\B_{j,\ell}}
\\ \nonumber
&\leq \frac{1}{4}\left(B_{M,k}(\ell+1)- B_{M,k}(\ell) \right)^2
\\ 
&=\frac{M^2}{ (k-1)^2},
\end{align}
where the inequality uses the GM-AM inequality and the final identity
is simply by the definition of $B_{M,k}(\ell)$.  Upon combining the
bounds above, we obtain \eq{
\E{\sum_{j\in[d]}\big(Q_{{\tt u}}(Y)(j)-{Y(j)}\big)^2\mathbbm{1}_{\{{|Y(j)|}\leq M\}}\mid Y} \leq
            \frac{dM^2}{
            (k-1)^2} \cdot \frac{1}{d}\sum_{j\in[d]} \mathbbm{1}_{\{{|Y(j)|}\leq
            M\}}.}  \qed

\paragraph{Proof of Lemma \ref{l:sup0m}}
The proof is the same as the proof of Lemma~\ref{l:sup}, except that
we need to set $d=1$ and replace the identity used
in\eqref{e:residual} with
\[
B_{M,k}(\ell+1)- B_{M,k}(\ell)=\frac{M}{k-1}.
\]

\subsection{Proof of Lemma \ref{l:concentration_a.s.}}\label{ap:R_subg}

For the rotation matrix $R=(1/\sqrt{d})HD$, each entry of $RY(j)$ of
the rotated matrix has the same distribution as $(1/\sqrt{d})V^T Y$,
where $V=[V(1), ..., V(d)]^{T}$ has independent Rademacher entries. We
will use this observation to bound the moment generating function of
$RY(i)$ conditioned on $Y$. Towards that end, we have
\eq{
\E{ e^{\lambda RY(i)}\mid Y}
& = \prod_{i=1}^{d} \E{ e^{\lambda V(i)Y(i)/\sqrt{d}} \mid Y}\\ &
= \prod_{i=1}^{d} \frac{e^{\lambda Y(i)/\sqrt{d}}+e^{- \lambda
Y(i)/\sqrt{d}}}{2}
\\
&\leq \prod_{i=1}^{d} e^{\lambda^2 Y(i)^2/2d}\\ &=
e^{\lambda^2 \norm{Y}_{2}^2/2d}, } where the first identity follows
from independence of $V(i)$s and the first inequality follows by the
fact that $(e^{x}+e^{-x})/2$ is less than $e^{x^2/2}$, which in turn
can be seen from the Taylor series expansion of these terms. Thus, we
have proved the following:
\begin{align}
\E{ e^{\lambda RY(i)}\mid Y} \leq e^{\lambda^2 \norm{Y}_2^2/2d}, \quad \forall \lambda\in \R, \forall i \in [d]. 
\label{e:subg}
\end{align}
 Note that $\norm{Y}_2^2$ can be further bounded by $B^2$, which along
 with \eqref{e:subg} leads to \[\E{ e^{\lambda RY(i)}} \leq
 e^{\lambda^2 B^2/2d} \quad \forall \lambda\in \R, \forall i \in
 [d]. \] Using this inequality and the observation that $\E{RY(i)}=0$,
 we note that $RY(i)$ is a centered subgaussian with a variance
 parameter $B^2/d$.  The second statement of the lemma trivially
 follows from Lemma \ref{l:standar_subg}.
\qed.

\subsection{Proof of Lemma~\ref{l:elemineq}}\label{app:elemineq}
For any $y \in [0,1]$ such that $ab^{y/(2+y)}< c$, we have
$ab^{y/(2+y)} < \min\left\{c, ab^{1/3} \right\}$.  By multiplying by
$a/b$ on both sides and taking square root, we get
\[
\frac{a}{b^{\frac{1}{2+y}}} < \min\left\{\sqrt{\frac{ca}{b}},\frac{a}{b^{1/3}}\right\},
\]
which gives
\[
\sup_{y \in [0,1]: a(b)^{y/(2+y)}< c.}\frac{a}{b^{\frac{1}{2+y}}} \leq \min\left\{\sqrt{\frac{ca}{b}},\frac{a}{b^{1/3}}\right\}.
\]
Making cases $ab^{1/3} \geq c$ and $ab^{1/3} < c$, we note that the
supremum on the left-side equals the right-side in both the cases.
\qed

\section{Analysis of distributed mean estimation}
 \subsection{Proof of
 Lemma \ref{l:MSE_bound}}\label{ap:dist_mean_1} \new{By the
 description of RATQ we have that $Q_{{\tt at}, R} =R^{-1} Q_{{\tt
 at}, I}(RY)$, where $Q_{{\tt at}, I}$ is as defined
 in~\eqref{e:Q_at_I}. Thus, using the fact that $R$ is a unitary
 matrix
\[
\E{ \norm{Q_{{\tt at}, R}(Y) -Y }_2^2 } =\E{ \norm{Q_{{\tt at}, I}(RY) -RY }_2^2}.
\]
 When the parameters are set as in~\eqref{e:RATQ_unit_levels}, we get
\[
RY(j) \leq M_{h-1} \text{ a.s., } \forall j \in [d],
\] 
whereby
\[
\E{ \norm{Q_{{\tt at}, R}(Y) -Y }_2^2 } =\E{ \sum_{j \in [d]}(Q_{{\tt at}, I}(RY)(j) -RY(j) )^2 \indic{RY(j) \leq M_{h-1}}}.    
\]
The proof is completed by noting that $Y$ satisfies $\norm{Y}_2\leq 1$
a.s., setting $m=3/d$ and $m_0=(2/d)\ln s$, and applying
Lemma~\ref{l:useful}.  } \qed

\subsection{Proof of Theorem \ref{t:RATQ_distributed_mean}}\label{ap:dist_mean_2}
 \begin{proof} When RATQ is employed by all the clients, the MSE
between the sample mean of the quantized vectors and the sample mean
of the input is bounded as
\eq{ \ep(\pi_{RATQ}, B^d)&= \max_{x: x_i \in B^d,\, \forall i\in [d] } 
\E{\norm{\frac{\sum_{i=1}^{n}Q_{{\tt at }, R_i}(x_i)}{n}-\frac{\sum_{i=1}^{n}{x_i}}{n}}_2^2} 
\\ &= \max_{x: x_i \in B^d,\, \forall i\in [d] } 
 \sum_{i=1}^{n}\frac{1}{n^2}\E{\norm{Q_{{\tt at },
R_i}(x_i)-x_i}_2^2} \\ & \leq \frac{9 +3\ln s}{n(k-1)^2}\\
&= \frac{1}{n}, } where the second identity uses the unbiasedness of
quantizers used by all the clients\footnote{Note that
$\{x_i\}_{i=1}^{n}$ are deterministic vectors, each with
$\norm{x_i}_2\leq 1$. Thus, setting $B=1$ and choosing $h$ so that
$M_{h-1}\geq 1$, we get the desired unbiased estimators.}  and the
independence of randomness used by the quantizers of different
clients, the third inequality uses Lemma~\eqref{l:MSE_bound}, and the
final identity follows by substituting for $k$.  \end{proof}

\section{Proof of Theorem \ref{t:gauss_rd}}
\label{ap:gauss_rd}
\newest{
We split the overall mean square error into two terms and 
derive upper bounds for each of them. Specifically, we have
  \eq{
  \frac{1}{d}\cdot \E{\norm{X_d-Q_d(X_d)}_2^2}&=\frac{1}{d}\cdot \E{\sum_{i \in [d]}(X_d(i)-Q_d(X_d)(i))^2  \indic{\{|X_d(i)|\leq M_{h-1}}\}}\\&\hspace{0.5cm}+\frac{1}{d}\cdot \E{\sum_{i \in [d]}(X_d(i)-Q_d(X_d)(i))^2  \indic{\{|X_d(i)| > M_{h-1}\}}}.
  }
The second term on the right-side above can be bounded as follows:
 \eq{\frac{1}{d}\cdot \E{\sum_{i \in [d]}(X_d(i)-Q_d(X_d)(i))^2  \indic{\{|X_d(i)| > M_{h-1}}\}}
& =  \frac{1}{d}\cdot \E{\sum_{i \in [d]}X_d(i)^2  \indic{\{|X_d(i)| > M_{h-1}}\}}\\
&\leq \E{X_d(1)^4}^{1/2} P\left(|X_d(1)| > M_{h-1}\right)^{1/2}\\
&\leq 4\sqrt{2} v e^{-\frac{M_{h-1}^2}{4v},}
} 
where the first inequality follows by the Cauchy-Schwarz inequality and the second follows by Lemma \ref{l:standar_subg}.
 Note that $M_{h-1}^2\geq m e^{*(h-1)}\geq 3 ve^{*\ln^*(4\ln(8\sqrt{2}v/D)/3)}=4v\ln(8\sqrt{2}v/D)$, which with the previous bound leads to
  \eq{\frac{1}{d}\cdot \E{\sum_{i \in [d]}(X_d(i)-Q_d(X_d)(i))^2  \indic{\{|X_d(i)| > M_{h-1}\}}}
& \leq \frac{D}{2}.}
Furthermore, by Lemma~\ref{l:subg_mse} we have
\eq{\frac{1}{d}\cdot \E{\sum_{i \in [d]}(X_d(i)-Q_d(X_d)(i))^2  \indic{\{|X_d(i)|\leq M_{h-1}}\}}\leq \frac{9v+3v\ln s}{(k-1)^2}
\leq\frac{D}{2}, }
where the last equality holds since
$k \geq 1+ \sqrt{\frac{18v+6v\ln s}{D}}.$

It remains to bound the rate. Note that the 
overall resolution used for the entire vector is
\eq{
  d \log (k+1)+\ceil{\frac{d}{s}}\log h &\leq  d \ceil{\log\left(2+\sqrt{\frac{18v+6v\ln s}{D}}\right)}+ d + \log h
}
Therefore, for $d\geq \log h$ and $D<v/4$, the proof is completed by bounding the rate $R$ as 
\eq{ R&\leq  \log\left(2+ \sqrt{\frac v D}\,\sqrt{{18+6\ln \log h}}\right)+ 3\\
&\leq \frac{1}{2} \log \frac{v}{D} +\log\left(1+ \sqrt{18+6\ln \ceil{\log\left(1+\ln^*\left(\frac{4\ln(8\sqrt{2}v/D)}{3}\right)\right)}}\right)+3\\
&\leq \frac{1}{2} \log \frac{v}{D}+ O\left(\log 
\log \log \log^*\log \frac{v}{D}\right).
}
}